\documentclass{article}
\usepackage{graphicx}
\usepackage{subcaption}
\usepackage{amsmath,amssymb,amsfonts,amsthm}
\usepackage{PRIMEarxiv}
\usepackage[utf8]{inputenc}
\usepackage[T1]{fontenc}
\usepackage{hyperref}
\usepackage{url}
\usepackage{booktabs}
\usepackage{nicefrac}
\usepackage{microtype}
\usepackage{algorithm}
\usepackage{algorithmic}
\usepackage{float}
\graphicspath{{media/}}

\theoremstyle{plain}
\newtheorem{theorem}{Theorem}
\newtheorem{lemma}{Lemma}
\newtheorem{proposition}{Proposition}
\newtheorem{corollary}{Corollary}
\newtheorem{assumption}{Assumption}

\theoremstyle{remark}
\newtheorem{remark}{Remark}

\title{Policy-Controlled Generalized Share: A General Framework with a Transformer Instantiation for Strictly Online Switching-Oracle Tracking}

\author{
  Hongkai Hu \\
  Antai College of Economics and Management \\
  Shanghai Jiao Tong University \\
  \texttt{if-only@sjtu.edu.cn}
}

\begin{document}
\pagestyle{plain}   
\maketitle

\begin{abstract}
Static regret to a single expert is often the wrong target for strictly online prediction under non-stationarity, where the best expert may switch repeatedly over time. We study \emph{Policy-Controlled Generalized Share} (PCGS), a general strictly online framework in which the generalized-share recursion is fixed while the post-loss update controls are allowed to vary adaptively. PCGS separates the analysis of the share backbone from the design of the controller. Its principal instantiation in this paper is \emph{PCGS-TF}, which uses a causal Transformer as an update controller: after round $t$ finishes and the loss vector is observed, the Transformer outputs the controls that map $w_t$ to $w_{t+1}$ without altering the already committed decision $w_t$. Under admissible $\mathcal{F}_t$-measurable update controls, we obtain a pathwise \emph{weighted} regret guarantee for general time-varying learning rates, and a standard dynamic-regret guarantee against any expert path with at most $S$ switches under the constant-learning-rate specialization. In both cases, the switching term depends directly on $\sum_{t:\,\pi_{t+1}\neq\pi_t} -\log q_t(\pi_{t+1})$, making restart allocation part of the certified difficulty of tracking a switching comparator. Empirically, all results labeled \textbf{PCGS-TF} correspond to this Transformer-controlled instantiation. On a controlled synthetic suite with exact DP switching-oracle evaluation, PCGS-TF attains the lowest mean dynamic regret in all seven non-stationary families and its advantage increases with larger expert pools. On a reproduced household-electricity benchmark, PCGS-TF also attains the lowest normalized dynamic regret for $S\in\{5,10,20\}$.
\end{abstract}


\section{Introduction}

\subsection{Problem setting and metric}

We study prediction with expert advice under a \emph{strictly online} protocol. At each round $t\in\{1,\dots,T\}$, the learner must commit to a probability distribution
\[
w_t\in\Delta_K
\qquad\text{where}\qquad
\Delta_K\triangleq\Big\{w\in\mathbb{R}_+^K:\sum_{k=1}^K w(k)=1\Big\},
\]
before observing the current loss vector
\[
\ell_t=(\ell_{t,1},\dots,\ell_{t,K})\in[0,1]^K.
\]
The incurred loss is the mixed loss
\[
\langle w_t,\ell_t\rangle=\sum_{k=1}^K w_t(k)\ell_{t,k},
\]
and the learner may update its internal state only after the round has been completed. This temporal ordering is essential: it rules out any use of $\ell_t$ in the construction of $w_t$ and therefore enforces a genuine online, rather than one-step-delayed offline, evaluation regime.

The classical benchmark in expert advice is static regret with respect to the single best expert in hindsight. That metric is appropriate when the environment is effectively stationary, or when one is willing to compare against a fixed action chosen after observing the entire sequence. In the present setting, however, the identity of the best expert may evolve over time: it may drift gradually, change abruptly at regime boundaries, or alternate repeatedly among different specialists. In such environments, comparison to a single static expert is often too coarse to capture the real difficulty of the task. A learner that tracks regime changes well may still exhibit large static regret simply because no single expert is uniformly best over the whole horizon.

For this reason, the natural comparator class is not a fixed expert but a \emph{switching path} of experts. Let $\pi_{1:T}\in[K]^T$ denote a comparator path, where $\pi_t$ is the expert chosen at time $t$, and let
\[
\#\mathrm{sw}(\pi)
=
\sum_{t=2}^T \mathbf{1}\{\pi_t\neq \pi_{t-1}\}
\]
be the number of switches in that path. For a switch budget $S\in\{0,\dots,T-1\}$, one considers the oracle value
\[
L_T^{\mathrm{sw}}(S)
\triangleq
\min_{\pi:\,\#\mathrm{sw}(\pi)\le S}\sum_{t=1}^T \ell_{t,\pi_t},
\]
and evaluates the learner through the corresponding dynamic regret
\[
\mathrm{DynRegret}_S
\triangleq
\sum_{t=1}^T \langle w_t,\ell_t\rangle - L_T^{\mathrm{sw}}(S).
\]
This metric explicitly quantifies the price of adaptation relative to a comparator that is allowed a controlled amount of non-stationarity. It is therefore the right target when the main challenge is not prediction against a fixed latent regime, but rapid tracking of a sequence of changing best experts.

\subsection{Limits of fixed restart schemes}

A canonical baseline for online expert aggregation is Hedge, or exponential weights, which is well known to enjoy optimal worst-case guarantees for static regret. Its appeal is undeniable: the update is simple, computationally efficient, and supported by a clean minimax theory. However, the geometry of multiplicative updates also induces a structural limitation in non-stationary environments. Once the weight assigned to an expert becomes very small, that expert can recover only gradually under subsequent multiplicative reweighting, even if it later becomes optimal. In other words, the posterior may become excessively committed to a regime that has already ended.

Share-type algorithms were introduced precisely to address this recoverability problem. Rather than relying solely on multiplicative reweighting, they inject a portion of probability mass back into the expert pool after each round. In Fixed Share, the amount of restarted mass is constant and the restart distribution is typically uniform. Generalized Share extends this basic mechanism by allowing greater flexibility in how restart mass is redistributed. The conceptual lesson is that, in changing environments, good performance depends not only on identifying which experts performed well on the current round, but also on maintaining the capacity to reallocate mass efficiently when the underlying regime changes.

This perspective immediately suggests that a fixed restart rule is unlikely to be uniformly adequate across heterogeneous non-stationary regimes. If the environment is locally stable, aggressive restart may introduce unnecessary diffusion and slow concentration. If the environment undergoes abrupt changes, insufficient restart can leave too much mass trapped on stale experts. Likewise, uniform restart destinations ignore potentially informative structure in the loss history: once evidence accumulates that certain experts are plausible ``next winners,'' it is suboptimal to spread all restart mass indiscriminately. Thus the main practical difficulty is not merely whether to restart, but \emph{how much} to restart, \emph{when} to restart, and \emph{where} to send the restarted mass.

These observations expose a gap between classical theory and practical algorithm design. Existing fixed or heuristic restart schemes capture the broad intuition that some form of forgetting is necessary, but they do not provide a principled way to adapt restart intensity and restart destination to the realized online history. Closing this gap is the main starting point of the present work.

\subsection{Policy-Controlled Generalized Share and PCGS-TF}

We study \emph{Policy-Controlled Generalized Share} (PCGS), a strictly online framework that retains the generalized-share backbone while allowing its update controls to be selected adaptively from observed history. The update takes the form
\[
\bar w_{t+1}(k)\propto w_t(k)\exp(-\eta_t\ell_{t,k}),
\qquad
w_{t+1}=(1-\rho_t)\bar w_{t+1}+\rho_t q_t,
\]
where $\eta_t$ is a learning-rate parameter, $\rho_t$ is a restart intensity, and $q_t\in\Delta_K$ is a restart distribution. The three quantities
\[
(\eta_t,\rho_t,q_t)
\]
constitute the \emph{post-loss update controls} of the algorithm.

The key structural distinction in this paper is between two levels of abstraction. At the general level, \textbf{PCGS} is the framework itself: it fixes the generalized-share recursion and allows the controls $(\eta_t,\rho_t,q_t)$ to be chosen by any admissible strictly causal rule. At the algorithmic level, the principal instantiation we study is \textbf{PCGS-TF}, in which a causal Transformer serves as the controller that outputs these post-loss update controls. This distinction is not cosmetic. The framework level is the correct object for theory: it is the generalized-share recursion with history-dependent, strictly online-admissible controls that admits regret analysis. The instantiation level is the correct object for the main method: it is the Transformer-controlled realization of that framework that is evaluated in the experiments.

It is crucial that the Transformer in PCGS-TF is \emph{not} a direct predictor of the round-$t$ target. Rather, it is a controller for the update map $w_t\mapsto w_{t+1}$. After round $t$ has completed and the loss vector $\ell_t$ has been observed, the Transformer processes the post-loss history and outputs $(\eta_t,\rho_t,q_t)$; these controls then determine how the current mixture is transformed into the next one. Since the already played decision $w_t$ is committed before $\ell_t$ is observed, strict online causality is preserved by construction. This role separation is central to the paper: the Transformer is used to control a theory-grounded online algorithm, not to bypass it.

The same formulation is also analytically advantageous. Generalized Share admits a path-space interpretation as exponential weights over expert trajectories under a time-inhomogeneous Markov prior. Under this view, the restart mechanism does not merely improve empirical adaptability; it enters the regret bound explicitly through the transition code length
\[
-\sum_{t=1}^{T-1}\log A_t(\pi_t\to \pi_{t+1}),
\]
where the transition kernel $A_t$ is induced by $(\rho_t,q_t)$. In particular, on switch steps one obtains a term of the form
\[
-\log \rho_t - \log q_t(\pi_{t+1}),
\]
so the restart distribution contributes directly to the certified difficulty of tracking a switching comparator. This makes restart allocation a first-class theoretical object rather than a secondary implementation heuristic. In this sense, the present framework aligns algorithm design, learning architecture, and regret analysis around the same transition-complexity perspective.

\subsection{Contributions}

Our contributions are fourfold.

\begin{enumerate}
    \item We formalize \textbf{Policy-Controlled Generalized Share (PCGS)} as a general strictly online framework in which the generalized-share backbone is fixed, while the update controls $(\eta_t,\rho_t,q_t)$ are allowed to be adapted from post-loss information in an $\mathcal{F}_t$-measurable manner. This formulation cleanly separates the online recursion from the controller and makes precise what it means to use learned update policies without violating strict online causality.

    \item We provide a pathwise regret analysis against arbitrary comparator paths. For general time-varying learning rates, the guarantee is stated in weighted-regret form; under the constant-learning-rate specialization, it yields the corresponding dynamic-regret guarantee for the $S$-switch oracle. The resulting bound makes explicit that the switching complexity depends on the restart allocation through the term $-\log q_t$, thereby showing that adaptive restart destinations are theoretically meaningful rather than merely empirically motivated.

    \item We instantiate the framework with a causal Transformer controller, \textbf{PCGS-TF}, and make explicit that the Transformer operates as a strictly causal update policy acting on the transition $w_t\mapsto w_{t+1}$. In particular, the Transformer is not used as a direct predictor of round-$t$ outcomes, but as a controller for the online update dynamics of a share-based expert-tracking algorithm.

    \item We evaluate PCGS-TF under a metric-matched strict-online protocol using exact dynamic-programming switching-oracle evaluation. The empirical study includes a controlled synthetic suite spanning several forms of non-stationarity and a reproduced household-electricity benchmark that we position as a reproducibility-oriented external-validity check rather than a broad real-world generalization claim.
\end{enumerate}

\section{Related Work}
\label{sec:related}

\paragraph{Online learning in changing environments.}
Online learning with expert advice and online convex optimization provide the foundational protocol and
analysis tools for strictly online sequential prediction \cite{CesaBianchiEtAl1997,CesaBianchiLugosi2006,Zinkevich2003,ShalevShwartz2012Survey,Hazan2016}.
Non-stationary extensions formalize time variation via switching, drift, or variation budgets and seek
dynamic/strongly-adaptive regret guarantees \cite{BesbesGurZeevi2014,ZhangChangingEnvSurvey2020}.
Recent work has sharpened the achievable rates and design space for dynamic regret in adversarial and
stochastic settings, including parameter-free and worst-case optimal guarantees, projection-free methods,
and results in dynamic non-convex environments \cite{Cutkosky2020,Zhaoworstcase2024,WangProjectionFree2024,XuZhangNeurIPS2024,LiEfficientNonStationary2025,BabyWangAISTATS2022}.
PCGS is complementary: we keep the backbone update within the classical exp-weights/share family while
\emph{learning} the restart distribution that directly controls the switching term in the regret bound.

\paragraph{Tracking the best expert and share-type algorithms.}
Tracking a switching oracle dates back to the seminal work of Herbster and Warmuth \cite{HerbsterWarmuth1998}
and its subsequent generalizations \cite{CesaBianchiEtAl1997}.
The core idea is to prevent irreversible elimination of experts by injecting restart mass, yielding bounds
that scale with the number of switches and $\log K$.
Our contribution is to expose and exploit the fact that the $\log K$ factor is not fundamental:
it arises from \emph{uniform} restart destinations.
By learning a state- and history-dependent restart distribution $q_t$, we replace $\log K$ with a
\emph{data-dependent} $\,-\log q_t(\cdot)$ term, aligning algorithm design, policy learning, and theory.

\paragraph{Transformers as in-context learners and algorithm controllers.}
Transformers were introduced as sequence models driven by attention \cite{Vaswani2017} and have since
been analyzed as \emph{in-context} or implicit meta-learning systems that can implement algorithmic
updates given the right prompting/histories \cite{GargICL2022,AkyurekICLR2023,vonOswaldICML2023}.
Empirical and survey work further clarifies which aspects of context drive ICL behavior and provides a
taxonomy of mechanisms \cite{MinEMNLP2022,DongSurveyICL2024}.
In reinforcement learning, policy Transformers (e.g., sequence modeling of trajectories) have popularized
the view of Transformers as \emph{policies} rather than predictors \cite{DecisionTransformer2021}.
PCGS brings this policy viewpoint into \emph{strictly-online} learning with regret guarantees: the
Transformer does not output predictions, but controls a provable online update through a small set of
interpretable actions.

\paragraph{Transformers for time series and the ``predictor'' paradigm.}
A large literature develops Transformer variants for long-horizon time series forecasting and related
tasks \cite{Informer2021,Autoformer2021,FEDformer2022,NonStationaryTransformers2022,PatchTST2023,TimesNet2023,iTransformer2024,WenSurveyTS2023}.
A notable counterpoint questions whether increasingly complex Transformer variants are necessary and
advocates strong linear baselines and careful evaluation \cite{ZengAreTransformersEffective2023}.
Our work is orthogonal to the architecture race: we use a Transformer as a \emph{controller} for a
theory-grounded online algorithm. In this view, the representational power of the Transformer is used to
infer restart decisions and destinations, not to directly emit forecasts.

\paragraph{Robustness under heavy tails and anomalies.}
Heavy-tailed noise, outliers, and abrupt jumps are classical stressors for sequential prediction; robust
procedures such as clipping, M-estimation, and median-of-means provide standard tools in statistics and
motivate modern robust online learning guarantees \cite{Catoni2012,NeurIPSHeavyTails2022,RobustTSF2024}.
Our use of bounded surrogate losses follows this tradition: the goal is not to claim robustness from a
separate distribution-shift benchmark, but to stabilize strictly online tracking and keep the regret
analysis well-posed under heavy-tail and jump perturbations.

\section{Problem Setup}
\label{sec:setup}

We formalize the strictly online expert-aggregation problem studied throughout the paper. The purpose of this section is threefold. First, it fixes the information structure and timing conventions under which causality is interpreted. Second, it specifies the loss construction used both by the algorithm and by the comparator class. Third, it introduces the switching-oracle benchmark and the corresponding dynamic-regret objective that will serve as the target metric in both theory and experiments.

\subsection{Strictly online protocol}

We consider a sequential prediction problem over a finite horizon of $T$ rounds, indexed by
\[
t\in[T]\triangleq \{1,\dots,T\}.
\]
Let $(\Omega,\mathcal{F},\mathbb{P})$ be an underlying probability space, and let
\[
\{\mathcal{F}_t\}_{t=0}^T
\]
be a filtration representing the information revealed to the learner up to time $t$. As usual in online learning, the key causal requirement is that the decision made at round $t$ must be based only on information available strictly before the current loss vector is revealed.

We allow, optionally, the presence of side information $x_t$ at round $t$. Depending on the application, $x_t$ may be available before the learner commits to its decision (in which case $x_t$ is $\mathcal{F}_{t-1}$-measurable), or it may be revealed together with other round-$t$ information (in which case it is $\mathcal{F}_t$-measurable). Our theoretical development does not rely on side information, but the policy that controls the update may consume it whenever it is available under the strict online protocol.

There are $K$ experts, indexed by
\[
k\in[K]\triangleq\{1,\dots,K\}.
\]
At each round $t$, the learner outputs a distribution over experts,
\[
w_t\in\Delta_K,
\qquad
\Delta_K \triangleq \Big\{w\in\mathbb{R}_+^K:\sum_{k=1}^K w(k)=1\Big\}.
\]
The strict online requirement is that $w_t$ be measurable with respect to the information available at decision time, namely $\mathcal{F}_{t-1}$ (together with $x_t$ if that side information is available before the prediction is made). In particular, $w_t$ may depend on the entire past history, but it may not depend on the current loss vector $\ell_t$.

The round-$t$ interaction proceeds as follows.
\begin{enumerate}
    \item \textbf{Decision.} The learner chooses a mixture $w_t\in\Delta_K$, subject to the strict online measurability constraint described above.

    \item \textbf{Loss revelation.} After the decision has been committed, the environment reveals a loss vector
    \[
    \ell_t=(\ell_{t,1},\dots,\ell_{t,K})\in[0,1]^K,
    \]
    where $\ell_t$ is $\mathcal{F}_t$-measurable but not $\mathcal{F}_{t-1}$-measurable in general.

    \item \textbf{Incurred loss.} The learner suffers the mixture loss
    \begin{equation}
    \ell_t(\mathrm{alg})
    \;=\;
    \langle w_t,\ell_t\rangle
    \;=\;
    \sum_{k=1}^K w_t(k)\,\ell_{t,k},
    \qquad
    L_T(\mathrm{alg})
    \;=\;
    \sum_{t=1}^T \langle w_t,\ell_t\rangle.
    \label{eq:alg_loss}
    \end{equation}

    \item \textbf{Post-loss update.} After observing $\ell_t$, the learner updates its internal state and constructs $w_{t+1}$ for the next round.
\end{enumerate}

In the \textsc{PCGS} framework, the post-loss update is parameterized by controls
\[
(\eta_t,\rho_t,q_t),
\]
where $\eta_t$ is a learning-rate parameter, $\rho_t$ is a restart intensity, and $q_t\in\Delta_K$ is a restart distribution over experts. These controls are produced \emph{after} the loss vector $\ell_t$ has been observed, and hence may depend on information available up to time $t$. Crucially, however, they only affect the transition from $w_t$ to $w_{t+1}$; they do not retroactively alter the already committed decision $w_t$. This timing convention is the sense in which the method is strictly online and strictly causal.

\paragraph{Causality convention.}
To avoid any ambiguity, we emphasize the distinction between \emph{decision-time information} and \emph{update-time information}. The mixture $w_t$ must be chosen before seeing $\ell_t$, whereas the control variables $(\eta_t,\rho_t,q_t)$ may use $\ell_t$ because they act only on the next-round state. Throughout the paper, any claim of strict online validity should be understood relative to this timing convention.

\subsection{Experts and loss construction}

The framework is agnostic to the internal form of the experts. In a given application, an expert may correspond to any base predictor or base strategy that is itself admissible under the information structure of the problem. Typical examples include lag-based predictors, moving-average forecasters, exponentially weighted moving averages, online linear predictors, or other sequential forecasting rules. Our focus is not on the internal estimation problem solved by each expert, but rather on the meta-level problem of online aggregation and tracking across experts under non-stationarity.

For the theoretical analysis, we work with bounded losses,
\[
\ell_{t,k}\in[0,1], \qquad \forall t\in[T],\; k\in[K].
\]
This is the natural regime for exponential-weights analysis and ensures that the one-step log-sum-exp control used later in the regret proof remains well behaved. In applications where the raw task loss is not naturally bounded in $[0,1]$---for example, squared prediction error under heavy-tailed noise or jump contamination---we replace it by a bounded surrogate obtained via scaling and clipping.

Concretely, if $\ell^{\mathrm{raw}}_{t,k}\ge 0$ denotes a raw loss quantity on the original scale, we may define
\[
\tilde \ell_{t,k}
\;=\;
\min\!\left\{\frac{\ell^{\mathrm{raw}}_{t,k}}{c},\,1\right\},
\]
for some scale parameter $c>0$. More generally, any measurable transformation that maps the raw losses into $[0,1]$ may be used, provided that the same transformed loss matrix is supplied both to the algorithm and to the comparator benchmark. To avoid excessive notation, we suppress the tilde in the sequel whenever the meaning is clear from context, and write simply $\ell_{t,k}$ for the bounded loss actually used by the algorithm and in the regret statements.

This bounded-loss construction serves two distinct purposes. First, it places the problem within the regime required by the exponential-weights/share analysis developed later. Second, it improves numerical stability in settings with heavy tails, transient spikes, or rare jump events, where unbounded losses would otherwise induce unstable multiplicative updates and highly volatile controller features.

\paragraph{Which loss the theory certifies.}
Whenever scaling or clipping is employed, the theoretical guarantees are statements about the realized bounded surrogate losses actually fed to the algorithm. In particular, the switching oracle, the reported dynamic regret, and the regret bounds in Section~\ref{sec:theory} are all defined relative to the same bounded loss matrix. Thus the analysis is internally consistent: the learner and the benchmark are evaluated on exactly the same transformed losses. Raw losses remain relevant for data preprocessing and interpretation, but they are not the objects certified directly by the regret bounds unless explicitly stated otherwise.

\subsection{Switching oracle and dynamic regret}

The central benchmark in this paper is not the best \emph{single} expert in hindsight, but rather the best \emph{switching path} of experts subject to a budget on the number of switches. This is the appropriate comparator in non-stationary environments, where the identity of the best expert may change repeatedly over time.

A comparator path is a sequence
\[
\pi_{1:T}=(\pi_1,\dots,\pi_T)\in[K]^T,
\]
where $\pi_t$ denotes the expert selected by the comparator at round $t$. The number of switches along the path is defined by
\begin{equation}
\#\mathrm{sw}(\pi)
\;\triangleq\;
\sum_{t=2}^T \mathbf{1}\{\pi_t\neq \pi_{t-1}\}.
\label{eq:num_switches}
\end{equation}
For a prescribed switch budget
\[
S\in\{0,1,\dots,T-1\},
\]
we define the admissible comparator class
\[
\Pi_S
\;\triangleq\;
\Big\{\pi\in[K]^T:\#\mathrm{sw}(\pi)\le S\Big\}.
\]

The corresponding switching-oracle value is
\begin{equation}
L_T^{\mathrm{sw}}(S)
\;\triangleq\;
\min_{\pi\in\Pi_S}\;
\sum_{t=1}^T \ell_{t,\pi_t}.
\label{eq:switch_oracle}
\end{equation}
That is, $L_T^{\mathrm{sw}}(S)$ is the cumulative loss of the best expert path in hindsight that is allowed to switch at most $S$ times. The associated dynamic regret of the learner against the $S$-switch oracle is then
\begin{equation}
\mathrm{DynRegret}_S
\;\triangleq\;
L_T(\mathrm{alg})-L_T^{\mathrm{sw}}(S).
\label{eq:dynregret}
\end{equation}

Several remarks are worth making. If $S=0$, then $\Pi_S$ contains only constant paths, and $\mathrm{DynRegret}_0$ reduces to the usual static regret against the best single expert in hindsight. For $S>0$, however, the benchmark becomes substantially stronger: it is allowed to reassign itself across experts as the environment evolves. This is precisely the regime of interest here, since the paper is concerned with tracking under non-stationarity rather than with static aggregation under stationarity.

\paragraph{Exact computation of the switching oracle.}
For a fixed bounded loss matrix $(\ell_{t,k})_{t\le T,k\le K}$ and switch budget $S$, the oracle value in \eqref{eq:switch_oracle} can be computed exactly by dynamic programming in time $O(TKS)$. The dynamic program keeps track of the best cumulative loss achievable up to time $t$, with at most $s$ switches, and ending at expert $k$. Using standard backpointers together with a best/second-best trick for the inner minimization over predecessor experts, one obtains both the oracle value and an optimizing path. In this paper, the exact DP oracle is used in two conceptually distinct roles:
\begin{enumerate}
    \item it defines the target evaluation metric $\mathrm{DynRegret}_S$ reported in the experiments; and
    \item it provides supervised trajectories for training the controller, namely the next oracle expert and the switch/no-switch indicator at each round.
\end{enumerate}
Because the oracle is computed from the realized loss matrix only after the sequence is complete, its use for training supervision or for evaluation does not violate the strict online protocol of the learner itself.

\paragraph{Where the policy fits the protocol.}
The controller used in \textsc{PCGS} sits strictly between round $t$ and round $t+1$. More precisely, after observing $\ell_t$, the policy outputs $(\eta_t,\rho_t,q_t)$ on the basis of the information available up to and including time $t$. These quantities then determine the update from $w_t$ to $w_{t+1}$. They do \emph{not} influence the already played distribution $w_t$. Consequently, the method remains strictly online by construction: decision-time causality is preserved, while the controller is still allowed to react immediately to newly observed losses when forming the next-round state.


\section{Method: PCGS Framework and PCGS-TF Instantiation}
\label{sec:method}

This section separates two objects that should remain distinct throughout the paper.

\paragraph{General framework (PCGS).}
\textbf{Policy-Controlled Generalized Share (PCGS)} denotes the strictly online generalized-share recursion together with an admissible post-loss control law for $(\eta_t,\rho_t,q_t)$. The framework is agnostic to how these controls are chosen.

\paragraph{Principal instantiation (PCGS-TF).}
\textbf{PCGS-TF} denotes the main method studied empirically in this paper: a causal Transformer controller coupled to the PCGS backbone. The Transformer does \emph{not} replace the expert-aggregation algorithm and does \emph{not} directly emit the round-$t$ prediction. Instead, after round $t$ is completed and $\ell_t$ has been observed, it outputs the controls used to update $w_t\mapsto w_{t+1}$. Therefore the already-committed decision $w_t$ remains untouched, and strict online causality is preserved by construction.

Unless stated otherwise, every empirical result reported for \textbf{PCGS-TF} refers to this Transformer-controlled instantiation.

\subsection{From Hedge to Share: why restarts are the right control knob}
\label{sec:hedge_to_share}

A natural starting point for expert aggregation is the classical exponentially weighted forecaster, often referred to as Hedge. Given a mixture $w_t\in\Delta_K$ and losses $\ell_t\in[0,1]^K$, the Hedge update assigns next-round weight to expert $k$ in proportion to
\[
w_t(k)\exp(-\eta_t \ell_{t,k}),
\]
thereby rewarding experts that perform well on the current round and penalizing those that perform poorly. This mechanism is minimax-optimal for static regret against the best \emph{single} expert in hindsight \cite{freund1997decision,cesabianchi2006prediction}, and it remains the canonical baseline for adversarial expert advice.

However, the static-regret viewpoint is mismatched to the non-stationary regime considered here. When the identity of the best expert changes over time, a purely multiplicative update suffers from a structural \emph{recoverability bottleneck}. To see the issue, suppose that at some past time an expert $k$ has accumulated a long sequence of losses that makes its weight $w_t(k)$ exponentially small. Even if expert $k$ subsequently becomes optimal, its future influence under pure multiplicative updating remains proportional to the vanishing quantity $w_t(k)$. In other words, the posterior can become \emph{over-committed} to an outdated regime, and once this happens, adaptation to a newly optimal expert may require many rounds. This phenomenon is not a looseness of the standard analysis; it is an intrinsic consequence of the multiplicative geometry of exponential weights.

The purpose of share-type methods is precisely to address this limitation. Instead of allowing the posterior to evolve only by multiplicative reweighting, one injects fresh probability mass into the expert pool after the loss update. In Fixed Share \cite{herbster1998tracking}, this is done by mixing the multiplicatively updated weights with a fixed restart distribution, typically uniform. The resulting update ensures that no expert can ever be completely eliminated: every expert retains a baseline amount of probability mass and can therefore re-enter the mixture when the environment changes.

From the viewpoint of switching comparators, this restart operation is the essential algorithmic degree of freedom. If one wishes to track a comparator path that may jump across experts, then it is not enough to merely penalize poorly performing experts less aggressively; one must also maintain the ability to \emph{reallocate mass across experts when a regime shift occurs}. This is why restart/share is the relevant control knob in the non-stationary setting: it governs how easily the learner can escape from an outdated posterior and transfer probability toward experts that become favorable only after the shift.

The \textsc{PCGS} framework generalizes this idea in the direction most relevant for strictly online adaptation. Rather than using a constant share rate and a fixed restart destination, we allow a policy to choose, after each round, both
\begin{enumerate}
    \item \emph{how much} probability mass should be restarted, and
    \item \emph{where} that restarted mass should be sent.
\end{enumerate}
Thus the restart mechanism is no longer a hand-designed static heuristic; it becomes a history-dependent control object that can respond to the observed online trajectory. This is the main conceptual step from classical share algorithms to policy-controlled generalized share.

\subsection{Generalized Share backbone with policy-controlled parameters}
\label{sec:genshare_backbone}

We now define the backbone update used throughout the paper. Let $w_t\in\Delta_K$ denote the mixture played at round $t$, and let $\tilde\ell_t=(\tilde\ell_{t,1},\dots,\tilde\ell_{t,K})\in[0,1]^K$ denote the bounded loss vector used by the algorithm. The first stage of the update is the usual multiplicative reweighting:
\begin{equation}
\bar w_{t+1}(k)
\;=\;
\frac{w_t(k)\exp\!\big(-\eta_t\,\tilde\ell_{t,k}\big)}
{\sum_{j=1}^K w_t(j)\exp\!\big(-\eta_t\,\tilde\ell_{t,j}\big)} ,
\qquad k\in[K],
\label{eq:mul_update}
\end{equation}
where $\eta_t>0$ is the round-$t$ learning-rate parameter. Equation~\eqref{eq:mul_update} may be viewed as the posterior that would have been obtained under ordinary Hedge if no explicit restart mechanism were introduced. The quantity $\eta_t$ controls the curvature of this multiplicative step: larger values of $\eta_t$ place more emphasis on recent loss differences, while smaller values preserve more of the pre-update mixture.

The second stage is a restart, or share, mixture:
\begin{equation}
\boxed{
w_{t+1}
\;=\;
(1-\rho_t)\,\bar w_{t+1} \;+\; \rho_t\,q_t,
}
\qquad
\rho_t\in(0,\rho_{\max}),\qquad q_t\in\Delta_K.
\label{eq:genshare}
\end{equation}
Here $\rho_t$ is the restart intensity and $q_t$ is the restart distribution. Thus the next-round weight vector is a convex combination of two terms:
\begin{enumerate}
    \item the multiplicatively updated posterior $\bar w_{t+1}$, which preserves continuity with the current posterior state, and
    \item the restart distribution $q_t$, which injects fresh mass into the expert pool in a direction chosen by the controller.
\end{enumerate}

This decomposition makes clear that the update contains two qualitatively different adaptation mechanisms. The multiplicative term reacts locally by reweighting experts according to their current losses, whereas the restart term reacts globally by redistributing probability mass independently of the current posterior proportions. In non-stationary settings, both mechanisms are needed: multiplicative updating alone may be too inertial after a large regime change, while restart without loss-sensitive reweighting would discard useful local performance information.

We define \textsc{Policy-Controlled Generalized Share (PCGS)} as the family of algorithms obtained by coupling \eqref{eq:mul_update} and \eqref{eq:genshare} with a strictly causal controller for the update parameters:
\[
(\eta_t,\rho_t,q_t)
\;=\;
\pi_\theta(\text{history up to time }t).
\]
Here the phrase ``history up to time $t$'' is to be understood in the sense of the strict online protocol established in Section~\ref{sec:setup}: the controller may use information available after round $t$ has been completed, including $\ell_t$, but the resulting controls affect only $w_{t+1}$ and not the already committed mixture $w_t$.

\paragraph{Framework versus instantiation.}
Equation~\eqref{eq:genshare} defines the general \textsc{PCGS} backbone. At the theoretical level, the controls $(\eta_t,\rho_t,q_t)$ may be any admissible strictly causal update controls satisfying the measurability and feasibility conditions stated later in Section~\ref{sec:theory}. At the algorithmic level, the principal instantiation studied in this paper is \textbf{PCGS-TF}, in which a causal Transformer parameterizes the controller and outputs the restart policy from the online history observed up through round $t$. This separation is important: the theoretical object is the policy-controlled generalized-share recursion itself, whereas the Transformer is one particular, and central, realization of the policy class.

\paragraph{Interpretation of the control variables.}
The update in \eqref{eq:mul_update}--\eqref{eq:genshare} exposes three distinct control channels.

\begin{enumerate}
    \item \textbf{Learning-rate control \boldmath$\eta_t$.}
    The parameter $\eta_t$ governs the aggressiveness of the multiplicative update. Large $\eta_t$ sharpens the posterior toward experts with lower current loss, while small $\eta_t$ dampens the immediate reaction and preserves mixture diversity. In this sense, $\eta_t$ controls the local sensitivity of the posterior to the most recent loss observation.

    \item \textbf{Restart-intensity control \boldmath$\rho_t$.}
    The parameter $\rho_t$ determines how strongly the learner discounts the multiplicatively updated posterior in favor of a fresh allocation. When $\rho_t$ is near zero, the update is close to ordinary Hedge; when $\rho_t$ is larger, the learner is explicitly encouraged to forget part of the current posterior and reinitialize mass elsewhere. Thus $\rho_t$ governs the degree of \emph{posterior reset} after observing round-$t$ losses.

    \item \textbf{Restart-destination control \boldmath$q_t$.}
    The distribution $q_t$ specifies where the restarted probability mass is placed. Uniform $q_t$ yields a noncommittal restart over all experts, while concentrated $q_t$ expresses a directional hypothesis about which experts are likely to become favorable next. In this sense, $q_t$ is the component that encodes the controller's belief about the geometry of the next regime.
\end{enumerate}

The importance of $q_t$ is not merely heuristic. As the regret analysis in Section~\ref{sec:theory} will show, the switching complexity of a comparator path depends explicitly on the term
\[
-\log q_t(\pi_{t+1})
\]
at switch times. Therefore the restart distribution enters the regret bound directly rather than indirectly. This makes $q_t$ a learnable proxy for the \emph{next-regime winner}: assigning large mass to the expert that the switching oracle moves to next reduces the transition code length paid in the bound. From this perspective, learning $q_t$ is not an auxiliary engineering choice but a mathematically meaningful way to reduce the certified difficulty of switching-oracle tracking.

\subsection{Strict causality and the policy interface}
\label{sec:policy_interface}

A central design requirement of \textsc{PCGS} is that the controller must be compatible with the strict online protocol introduced in Section~\ref{sec:setup}. In particular, the learner is required to commit to the round-$t$ mixture
\[
w_t\in\Delta_K
\]
before the current loss vector $\ell_t$ is revealed. Consequently, any quantity that influences the played decision at time $t$ must be measurable with respect to the information available strictly before observing $\ell_t$. By contrast, the update controls used to form $w_{t+1}$ may depend on the enlarged information set available \emph{after} the loss vector has been observed.

Formally, the timing convention is
\[
w_t \text{ is }\mathcal{F}_{t-1}\text{-measurable},
\qquad
(\eta_t,\rho_t,q_t)\text{ is }\mathcal{F}_t\text{-measurable}.
\]
Equivalently, the decision map and the update map live on two different information sets:
\begin{enumerate}
    \item the played mixture $w_t$ is chosen from pre-loss information only;
    \item the update controls $(\eta_t,\rho_t,q_t)$ are computed from post-loss information and act only on the transition from $w_t$ to $w_{t+1}$.
\end{enumerate}
This separation is the precise sense in which the controller is \emph{strictly causal}. The controller is allowed to react immediately to newly observed losses, but only by shaping the next-round state; it never modifies the already executed action. Thus the architecture respects the standard online-learning causality constraint while still allowing rich post-loss adaptation.

It is useful to distinguish two conceptually different roles played by the controller. First, it acts as a \emph{meta-update rule}: given the realized online history up to time $t$, it chooses how aggressively the multiplicative posterior should be updated and how much fresh mass should be injected into the expert pool. Second, it acts as a \emph{restart allocator}: if a restart is performed, it determines which experts should receive the restarted mass. Both roles are causal in the strict sense above, because they only affect future mixtures.

\paragraph{Why this interface is natural.}
The policy interface in \textsc{PCGS} matches the timing of standard implementations of online learning algorithms. In many online procedures, one first incurs the round-$t$ loss and only then updates internal statistics, learning rates, trust-region parameters, or other control variables before entering round $t+1$. The present setup simply makes that timing explicit and elevates these post-loss update choices to learned control objects. In particular, the framework does \emph{not} relax strict online evaluation by allowing the controller to peek at $\ell_t$ before $w_t$ is played; rather, it learns a history-dependent update rule that remains fully compatible with the online protocol.

\paragraph{Feasibility, positivity, and numerical safety.}
Because the update controls enter directly into the generalized-share recursion and into the regret analysis, they must satisfy nontrivial feasibility constraints at every round. We therefore parameterize the raw controller outputs through smooth maps that enforce these constraints automatically:
\[
\rho_t = \rho_{\max}\,\sigma(r_t),
\qquad
q_t = (1-\varepsilon)\,\mathrm{softmax}(s_t)+\varepsilon\,\frac{\mathbf{1}}{K},
\qquad
\eta_t = \eta_{\min}+\mathrm{softplus}(e_t),
\]
where $r_t\in\mathbb{R}$, $s_t\in\mathbb{R}^K$, and $e_t\in\mathbb{R}$ are unconstrained controller outputs, $\sigma(u)=(1+e^{-u})^{-1}$ is the sigmoid, and $\mathrm{softplus}(u)=\log(1+e^u)$.

These parameterizations enforce the following properties:
\begin{enumerate}
    \item \textbf{Restart intensity.} Since $\sigma(r_t)\in(0,1)$, we obtain
    \[
    0<\rho_t<\rho_{\max},
    \]
    which ensures that the restart intensity is strictly positive but uniformly bounded away from $1$ by construction.

    \item \textbf{Restart distribution.} Because $\mathrm{softmax}(s_t)\in\Delta_K$ and $\mathbf{1}/K\in\Delta_K$, we have
    \[
    q_t\in\Delta_K.
    \]
    Moreover, the $\varepsilon$-uniform mixing implies the coordinatewise lower bound
    \[
    q_t(k)\ge \varepsilon/K,\qquad \forall k\in[K].
    \]
    This lower bound is important for two reasons: it prevents numerical degeneracy in the restart distribution, and it guarantees that logarithmic quantities such as $-\log q_t(k)$ remain finite, which is essential for the transition-code interpretation used in the regret analysis.

    \item \textbf{Learning rate.} Since $\mathrm{softplus}(e_t)>0$, we obtain
    \[
    \eta_t>\eta_{\min}>0,
    \]
    which ensures that the multiplicative update remains well defined and that the lower bound $\eta_t\ge \eta_{\min}$ needed in the pathwise regret bound is automatically satisfied.
\end{enumerate}

From an optimization perspective, these smooth parameterizations are also convenient because they preserve end-to-end differentiability of the controller while respecting the geometry of the feasible set. In particular, the simplex constraint on $q_t$, the positivity of $\eta_t$, and the boundedness of $\rho_t$ need not be enforced by projection or ad hoc truncation during training.

\paragraph{Interpretation of the controller outputs.}
Under this interface, the controller emits three different types of post-loss instructions:
\begin{enumerate}
    \item a scalar \emph{aggressiveness} parameter $\eta_t$ for the multiplicative step;
    \item a scalar \emph{restart propensity} $\rho_t$ governing how strongly the learner should discount the multiplicatively updated posterior;
    \item a simplex-valued \emph{restart direction} $q_t$ specifying where fresh mass should be allocated.
\end{enumerate}
The first quantity modulates local posterior sensitivity to the latest observed losses. The second quantity modulates the degree of posterior reset. The third quantity determines the directional content of that reset. This decomposition is not merely algorithmically convenient; it is also analytically meaningful, because the regret bound will depend explicitly on the transition probabilities induced by $(\rho_t,q_t)$.

\subsection{PCGS-TF: causal Transformer controller over expert tokens}
\label{sec:policy_transformer}

We now describe the principal instantiation of the controller, denoted \textbf{PCGS-TF}. The guiding principle is that the controller should operate not as a direct forecaster of the round-$t$ target, but rather as a \emph{strictly causal algorithmic policy} that maps the recent online history into update controls for the generalized-share backbone.

Let $\mathcal{H}_t$ denote the information available after round $t$ has been completed, including the bounded loss history observed up to time $t$ and any additional admissible side information. The controller processes $\mathcal{H}_t$ through an expert-centric representation. Specifically, for each expert $k\in[K]$, we form a token
\[
\mathrm{token}_t(k)
\;=\;
\phi\!\big(\tilde\ell_{t-L+1:t,k}\big)\in\mathbb{R}^d,
\]
where $L$ is a window length and
\[
\tilde\ell_{t-L+1:t,k}
\;=\;
(\tilde\ell_{t-L+1,k},\dots,\tilde\ell_{t,k})
\]
denotes the recent bounded-loss trajectory of expert $k$. The feature map $\phi$ is strictly causal: it uses only information revealed through round $t$. Depending on the implementation, $\phi$ may include quantities such as the most recent loss, window averages, empirical trends, local dispersion, extrema, exponentially weighted summaries, or other statistics of the expert's recent behavior. The role of $\phi$ is not to replace sequence modeling, but to provide the controller with a compact, expert-wise state representation that can be processed jointly across the expert pool.

\paragraph{Why expert tokens are the right state representation.}
The update problem faced by the controller is inherently relational. The controller does not merely need to assess whether a single expert is good or bad in isolation; it must decide how to redistribute mass across the entire expert pool. This depends on relative patterns such as:
\begin{enumerate}
    \item which experts are currently outperforming others;
    \item which experts have recently improved or deteriorated;
    \item whether the recent loss geometry suggests a regime shift;
    \item whether the currently dominant posterior should be trusted or partially reset.
\end{enumerate}
An expert-token representation is natural for this purpose because it allows the controller to compare experts jointly and to infer relational structure among them.

In \textbf{PCGS-TF}, these tokens are processed by a causal Transformer-style encoder. Denoting the contextualized hidden representation of expert $k$ at round $t$ by
\[
h_t(k)=\mathrm{enc}\big(\{\mathrm{token}_t(j)\}_{j=1}^K\big)_k,
\]
the controller then applies lightweight output heads to obtain the raw control logits. In particular, the restart logits are computed as
\[
s_t(k)=\mathrm{head}_q\!\big(h_t(k)\big),
\]
and the restart-intensity score is computed from a pooled summary of the encoded token set,
\[
r_t=\mathrm{head}_\rho\!\big(\mathrm{pool}(\{h_t(j)\}_{j=1}^K)\big).
\]
If $\eta_t$ is also controlled adaptively, one may analogously define
\[
e_t=\mathrm{head}_\eta\!\big(\mathrm{pool}(\{h_t(j)\}_{j=1}^K)\big),
\]
after which $(r_t,s_t,e_t)$ are transformed into $(\rho_t,q_t,\eta_t)$ by the feasibility maps described in Section~\ref{sec:policy_interface}.

The role of self-attention here is to provide the controller with a permutation-sensitive yet globally contextual view of the current expert landscape. Each expert token can be interpreted in light of the others, allowing the controller to distinguish, for example, between isolated noise spikes and coherent regime-wide changes. This is precisely the type of global pattern recognition required for deciding whether a restart should occur and, if so, where restarted probability mass should be directed.

\paragraph{Transformer as controller, not predictor.}
It is important to emphasize again that the Transformer in \textbf{PCGS-TF} is \emph{not} used as an end-to-end predictor of the current target variable. Its output does not directly replace the expert mixture $w_t$, nor does it emit a forecast that competes with the experts. Instead, it implements a history-dependent policy over the \emph{update space} of the generalized-share algorithm. Put differently, the Transformer acts on the law of motion of the posterior rather than on the prediction target itself. This distinction is conceptually central to the paper: the innovation lies in learning how to control a provable online algorithm, not in replacing that algorithm by a black-box predictor.

\paragraph{Framework versus instantiation.}
The distinction between the general framework and its principal instantiation should be kept explicit. At the framework level, \textsc{PCGS} refers to the generalized-share recursion equipped with arbitrary admissible strictly causal controls. At the implementation level, \textbf{PCGS-TF} refers to the specific choice of a causal Transformer controller that maps post-loss histories into those controls. Thus, all empirical results labeled \textbf{PCGS-TF} in Section~\ref{sec:experiments} correspond exactly to this Transformer-controlled update policy applied to the \textsc{PCGS} backbone.

\subsection{Training: supervising the Transformer controller with switching-oracle trajectories}
\label{sec:training_oracle}

A key methodological choice is to train the \textbf{PCGS-TF} controller using \emph{switching-oracle trajectories} computed offline on training sequences. The reason for doing so is directly tied to the theory: the terms controlled by the controller---in particular the restart intensity and the restart destination---are precisely the terms that govern switching-oracle tracking in the pathwise regret bounds of Section~\ref{sec:theory}. Thus the controller is trained not against an unrelated auxiliary target, but against supervision that is naturally aligned with the target comparator class.

Fix a training sequence and a switch budget $S$. Using the exact dynamic program associated with \eqref{eq:switch_oracle}, we compute an oracle path
\[
\pi^\star_{1:T}\in\arg\min_{\pi\in\Pi_S}\sum_{t=1}^T \tilde\ell_{t,\pi_t}.
\]
This path identifies, retrospectively, the best sequence of experts with at most $S$ switches under the same bounded losses used by the learner. From this oracle path we extract two supervised signals at each time $t\in\{1,\dots,T-1\}$.

\paragraph{$q$-head supervision: next-expert prediction.}
The restart distribution $q_t$ is trained to place mass on the expert that the oracle selects at the next time step. Concretely, the target for the $q$-head at time $t$ is $\pi^\star_{t+1}$, and the corresponding loss is the cross-entropy
\begin{equation}
\mathcal{L}_q(t)
\;=\;
-\log q_t(\pi^\star_{t+1}).
\label{eq:lq}
\end{equation}
This objective encourages the controller to allocate restart mass toward the expert that the switching oracle will occupy at the next step. Since switch times are the points at which restart allocation matters most, \eqref{eq:lq} may be viewed as a supervised proxy for reducing the switch-dependent code length that appears in the regret bound.

\paragraph{$\rho$-head supervision: switch-propensity prediction.}
The restart intensity should be large when the current posterior is likely to be stale and a switch in the comparator path is imminent, and small when the current regime is stable. Accordingly, we define the binary switch indicator
\[
s_t
\;=\;
\mathbf{1}\{\pi^\star_{t+1}\neq \pi^\star_t\}.
\]
Let
\[
p_t=\sigma(r_t)\in(0,1)
\]
be the controller's estimated switch propensity, and set
\[
\rho_t=\rho_{\max}p_t.
\]
We then supervise $p_t$ using binary cross-entropy:
\begin{equation}
\mathcal{L}_\rho(t)
\;=\;
-s_t\log p_t-(1-s_t)\log(1-p_t).
\label{eq:lrho}
\end{equation}
This term encourages the controller to predict when a reset of the posterior is likely to be beneficial according to the oracle path.

\paragraph{Combined objective.}
The total controller loss aggregates the two terms over time and adds standard $\ell_2$ regularization:
\begin{equation}
\mathcal{L}
\;=\;
\sum_{t=1}^{T-1}
\Big(
\mathcal{L}_q(t)+\lambda_{\mathrm{sw}}\mathcal{L}_\rho(t)
\Big)
+
\lambda_{\mathrm{wd}}\|\theta\|_2^2.
\label{eq:policy_loss}
\end{equation}
The coefficient $\lambda_{\mathrm{sw}}$ balances the relative emphasis placed on next-expert allocation and switch-propensity prediction, while $\lambda_{\mathrm{wd}}$ controls weight decay.

\paragraph{Why this supervision is theoretically aligned.}
The alignment is structural rather than merely heuristic. In the pathwise regret bound, the transition complexity of a comparator path involves terms of the form
\[
-\log A_t(\pi_t\to \pi_{t+1}),
\]
and on switch steps this decomposes into
\[
-\log \rho_t - \log q_t(\pi_{t+1}).
\]
Thus the very quantities that determine the difficulty of tracking the switching oracle are precisely the quantities targeted by $\mathcal{L}_\rho$ and $\mathcal{L}_q$. The training objective in \eqref{eq:policy_loss} is therefore a data-driven surrogate for the same code-length terms that appear in the regret certificate. In particular, the $q$-head cross-entropy upper bounds the empirical switch-dependent code length along the oracle path, as formalized in Lemma~\ref{lem:ce_switch}.

\paragraph{Strict online validity of the supervision protocol.}
There is no conflict between strict online evaluation and oracle-based supervision. The oracle path is computed offline on completed training sequences and is used only to train the controller parameters across tasks or sequences. At test time, the learned controller still operates strictly online: at round $t$ it uses only information revealed up to time $t$ to output controls for the update to $w_{t+1}$. Hence the oracle is a supervision device, not an online source of privileged information.

\subsection{Robust clipping for heavy tails and jumps}
\label{sec:robust_clip}

Many sequential prediction problems exhibit heavy-tailed losses, transient outliers, or abrupt jump events. In such settings, raw loss values may be poorly matched to multiplicative updating: a single extremely large loss can dominate the exponentiated weights, destabilize the posterior, and distort the controller's state representation. Since our theoretical analysis is formulated for bounded losses in $[0,1]$, we therefore apply a bounded surrogate whenever the raw loss scale is not naturally normalized.

A representative construction is
\[
\tilde\ell_{t,k}
=
\min\!\left\{\frac{\ell_{t,k}}{c},\,1\right\},
\]
for some scale parameter $c>0$, although any equivalent measurable scaling-and-clipping rule may be used. The resulting bounded losses serve simultaneously as:
\begin{enumerate}
    \item the inputs to the multiplicative update;
    \item the inputs to the controller features;
    \item the losses used by the switching oracle; and
    \item the losses certified by the regret analysis.
\end{enumerate}

This preprocessing plays two distinct roles. The first is theoretical: it makes explicit that the assumptions underlying the exponential-weights analysis are satisfied. The second is algorithmic: it prevents rare but extreme losses from inducing unstable posterior collapse or highly erratic controller behavior. In particular, the bounded surrogate makes the dynamics of both the generalized-share backbone and the Transformer controller substantially more regular under heavy tails and jump contamination.

\subsection{Strictly online execution order}
\label{sec:online_loop}

Algorithm~\ref{alg:pcgs} summarizes the operational flow of \textsc{PCGS} under the strict online protocol. The most important point is the temporal ordering: the played mixture $w_t$ is committed before observing $\ell_t$, whereas the controller is evaluated only after the loss has been revealed and therefore influences only the next-round state.

\begin{algorithm}[t]
\caption{PCGS strictly online loop}
\label{alg:pcgs}
\begin{algorithmic}[1]
\STATE Initialize $w_1=\mathbf{1}/K$ and initialize the controller state.
\FOR{$t=1$ to $T$}
    \STATE \textbf{Play} the mixture $w_t$ using only information available up to time $t-1$.
    \STATE Observe the current loss vector $\ell_t$ and incur mixture loss $\langle w_t,\ell_t\rangle$.
    \STATE Construct bounded losses $\tilde\ell_t$ via scaling/clipping if required.
    \STATE Form the post-loss history $\mathcal{H}_t$ and feed it to the controller.
    \STATE The controller outputs $(\eta_t,\rho_t,q_t)$.
    \STATE Compute the multiplicative posterior $\bar w_{t+1}$ via \eqref{eq:mul_update}.
    \STATE Form the next mixture $w_{t+1}$ via the restart update \eqref{eq:genshare}.
\ENDFOR
\end{algorithmic}
\end{algorithm}

\paragraph{Operational interpretation.}
The algorithm may be read as the composition of two maps at each round:
\begin{enumerate}
    \item a \emph{loss-sensitive multiplicative map} from $(w_t,\tilde\ell_t,\eta_t)$ to $\bar w_{t+1}$; and
    \item a \emph{restart map} from $(\bar w_{t+1},\rho_t,q_t)$ to $w_{t+1}$.
\end{enumerate}
The controller intervenes only through the choice of the update controls for these maps. It does not directly overwrite the posterior and does not bypass the generalized-share recursion. This is precisely why the method retains both a clear algorithmic interpretation and a regret analysis grounded in the structure of share-type updates.

\section{Theory}
\label{sec:theory}

This section formalizes the regret guarantees for \textbf{Policy-Controlled Generalized Share (PCGS)} and clarifies why learning the restart distribution $q_t$ is theoretically meaningful rather than a purely heuristic design choice. The theory is stated at the level of the \emph{general PCGS framework}; the Transformer-controlled method \textbf{PCGS-TF} is the principal instantiation used in the experiments.

\subsection{Paths-as-experts viewpoint and induced Markov prior}
\label{sec:path_prior}

Generalized Share admits a path-based interpretation as exponential weights over the enlarged class of expert paths $\pi_{1:T}\in[K]^T$ equipped with a time-varying Markov prior. Given restart controls $(\rho_t,q_t)$, define the transition kernel
\begin{equation}
A_t(i\to j) \;=\; (1-\rho_t)\mathbf{1}\{i=j\} + \rho_t q_t(j),
\qquad t=1,\dots,T-1.
\label{eq:At}
\end{equation}
This induces a prior over expert paths,
\begin{equation}
\mathbb{P}(\pi_{1:T})
\;=\;
w_1(\pi_1)\prod_{t=1}^{T-1} A_t(\pi_t\to\pi_{t+1}),
\label{eq:path_prior}
\end{equation}
and the generalized-share update is exactly the marginal posterior recursion associated with this prior. This view is useful because the regret bound can be expressed directly in terms of the transition code length of the comparator path.

\subsection{Admissible strictly causal update controls}
\label{sec:admissible_controls}

We call a control sequence $\{(\eta_t,\rho_t,q_t)\}_{t=1}^{T-1}$ \emph{admissible} if, for each $t$,
\begin{enumerate}
    \item $\eta_t>0$;
    \item $\rho_t\in[0,1)$ and $q_t\in\Delta_K$;
    \item $(\eta_t,\rho_t,q_t)$ is $\mathcal{F}_t$-measurable.
\end{enumerate}
These controls are therefore \emph{strictly causal update controls}: they may depend on the history observed through round $t$, including $\ell_t$, but they affect only the update from $w_t$ to $w_{t+1}$ and never the already played decision $w_t$.

\paragraph{Weighted versus unweighted regret statements.}
For general time-varying learning rates $\{\eta_t\}$, the clean pathwise statement is a \emph{weighted} regret bound in which each round-$t$ regret term is multiplied by $\eta_t$. The familiar unweighted regret and dynamic-regret forms are recovered under the constant-learning-rate specialization $\eta_t\equiv\eta>0$. In particular, a pointwise lower bound such as $\eta_t\ge \eta_{\min}>0$ is useful for implementation and numerical safety, but by itself it does \emph{not} justify converting a weighted-regret bound into an unweighted one when $\eta_t$ varies over time.

\subsection{Main theorem: pathwise regret under admissible controls}
\label{sec:main_theorem}

\begin{theorem}[Weighted pathwise regret for PCGS under admissible strictly causal update controls]
\label{thm:pcgs_pathwise}
Assume $\tilde\ell_{t,k}\in[0,1]$ for all $t\in[T]$ and $k\in[K]$. Let $w_1\in\Delta_K$ be $\mathcal{F}_0$-measurable, and run PCGS with admissible controls $\{(\eta_t,\rho_t,q_t)\}_{t=1}^{T-1}$ as defined in Section~\ref{sec:admissible_controls}. Define
\[
A_t(i\to j)=(1-\rho_t)\mathbf{1}\{i=j\}+\rho_t q_t(j),\qquad t=1,\dots,T-1.
\]
Then the played weights $w_t$ are $\mathcal{F}_{t-1}$-measurable for every $t$, and for every realized loss sequence and every comparator path $\pi_{1:T}\in[K]^T$,
\begin{equation}
\sum_{t=1}^T \eta_t\Big(\langle w_t,\tilde\ell_t\rangle-\tilde\ell_{t,\pi_t}\Big)
\;\le\;
-\log w_1(\pi_1)
-
\sum_{t=1}^{T-1}\log A_t(\pi_t\to\pi_{t+1})
+
\frac{1}{8}\sum_{t=1}^T \eta_t^2.
\label{eq:thm_tv_eta}
\end{equation}
In particular, under the constant-learning-rate specialization $\eta_t\equiv\eta>0$,
\begin{equation}
\sum_{t=1}^T \langle w_t,\tilde\ell_t\rangle
-
\sum_{t=1}^T \tilde\ell_{t,\pi_t}
\;\le\;
\frac{1}{\eta}
\Big[
-\log w_1(\pi_1)
-
\sum_{t=1}^{T-1}\log A_t(\pi_t\to\pi_{t+1})
\Big]
+\frac{\eta T}{8}.
\label{eq:thm_constant_eta}
\end{equation}
\end{theorem}

\paragraph{Interpretation.}
Theorem~\ref{thm:pcgs_pathwise} is a \emph{pathwise} statement: it holds for the realized sequence and does not require stochastic assumptions such as independence or stationarity. For general time-varying learning rates, the theorem yields a weighted-regret guarantee. The familiar unweighted regret form is recovered in the constant-learning-rate case. In both forms, the role of the controller enters only through the admissible update controls and the induced transition code length.

\paragraph{Proof sketch.}
Apply the one-step log-sum-exp inequality with parameter $\eta_t$ and sum over time to obtain a telescoping bound on a path partition function. Lower bound that partition function by any comparator path to obtain the weighted inequality \eqref{eq:thm_tv_eta}. The constant-learning-rate bound \eqref{eq:thm_constant_eta} is then immediate by dividing through by $\eta$. Full details are given in Appendix~\ref{app:proof_main}.

\subsection{Switching penalty becomes data-dependent}
\label{sec:data_dependent_switch}

Theorem~\ref{thm:pcgs_pathwise} makes the key mechanism transparent. At time $t$,
\[
A_t(\pi_t\to\pi_{t+1})=
\begin{cases}
(1-\rho_t)+\rho_t q_t(\pi_t), & \pi_{t+1}=\pi_t,\\
\rho_t q_t(\pi_{t+1}), & \pi_{t+1}\neq \pi_t.
\end{cases}
\]
Therefore each switch step contributes
\begin{equation}
-\log A_t(\pi_t\to\pi_{t+1})
=
-\log \rho_t - \log q_t(\pi_{t+1}),
\qquad (\pi_{t+1}\neq \pi_t).
\label{eq:switch_cost}
\end{equation}
Under uniform restarts, this reduces to the familiar $\log K$ contribution. Under learned restarts, it becomes the data-dependent code length $-\log q_t(\pi_{t+1})$. This is the core theoretical reason to learn $q_t$: restart allocation directly changes the certified difficulty of tracking a switching comparator.

\begin{figure}[t]
  \centering
  \includegraphics[width=0.92\linewidth]{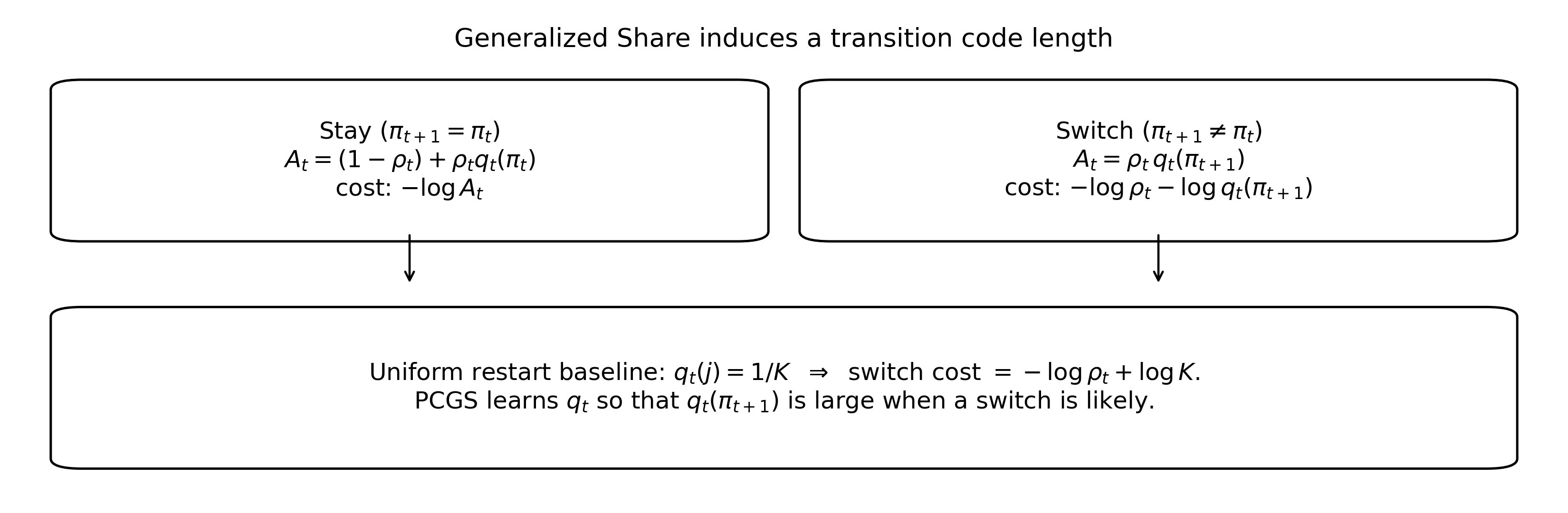}
  \caption{\textbf{Switching complexity becomes learnable.}
  Under PCGS, switching to expert $\pi_{t+1}$ incurs
  $-\log(\rho_t q_t(\pi_{t+1}))$, replacing the classical uniform-restart cost $\log K$ by
  a data-dependent code length $-\log q_t(\pi_{t+1})$.}
  \label{fig:switch_complexity}
\end{figure}

\subsection{Corollaries: switching-oracle regret and Fixed Share as a special case}
\label{sec:corollaries}

\begin{corollary}[Dynamic regret against the exact $S$-switch oracle]
\label{cor:dynregret_switch}
Under the conditions of Theorem~\ref{thm:pcgs_pathwise}, and additionally under the constant-learning-rate specialization $\eta_t\equiv\eta>0$, for every switch budget $S\in\{0,\dots,T-1\}$,
\begin{equation}
\mathrm{DynRegret}_S
\;\le\;
\frac{1}{\eta}
\min_{\pi\in\Pi_S}
\Big[
-\log w_1(\pi_1)
-
\sum_{t=1}^{T-1}\log A_t(\pi_t\to\pi_{t+1})
\Big]
+
\frac{\eta T}{8}.
\label{eq:cor_dynregret_switch}
\end{equation}
In particular, for any oracle path $\pi^\star\in\arg\min_{\pi\in\Pi_S}\sum_{t=1}^T\tilde\ell_{t,\pi_t}$, every switch time contributes $-\log \rho_t-\log q_t(\pi^\star_{t+1})$ to the complexity term.
\end{corollary}

\begin{proof}
Apply the constant-learning-rate specialization \eqref{eq:thm_constant_eta} of Theorem~\ref{thm:pcgs_pathwise} to each path $\pi\in\Pi_S$, and then minimize the resulting upper bound over $\Pi_S$.
\end{proof}

\begin{corollary}[Fixed Share as a special case]
\label{cor:fixedshare}
Assume additionally that $w_1\equiv \mathbf{1}/K$, $\eta_t\equiv\eta$, $\rho_t\equiv \rho\in(0,1)$, and $q_t\equiv \mathbf{1}/K$. Then for any path $\pi\in\Pi_S$,
\begin{equation}
\sum_{t=1}^T \langle w_t,\tilde\ell_t\rangle
-
\sum_{t=1}^T \tilde\ell_{t,\pi_t}
\;\le\;
\frac{1}{\eta}
\Big[
\log K + S\log\frac{K}{\rho} + (T-1)\log\frac{1}{1-\rho}
\Big]
+\frac{\eta T}{8}.
\label{eq:cor_fixedshare}
\end{equation}
Choosing $\rho\asymp S/T$ recovers the classical $S\log K$-type dependence up to standard lower-order factors.
\end{corollary}

\subsection{Why oracle-supervised training matches the theory}
\label{sec:ce_link}

The controller is trained to predict the next oracle expert and the switch indicator. This is theoretically meaningful because the same quantities govern the transition code length in Theorem~\ref{thm:pcgs_pathwise}.

\begin{lemma}[Cross-entropy training upper bounds the empirical switching-complexity term]
\label{lem:ce_switch}
Fix an oracle path $\pi^\star_{1:T}$. Define the effective switching complexity
\[
C_{\mathrm{eff}}(\pi^\star)
\triangleq
\sum_{t:\,\pi^\star_{t+1}\neq \pi^\star_t} -\log q_t(\pi^\star_{t+1}).
\]
Then
\[
C_{\mathrm{eff}}(\pi^\star)
\le
\sum_{t=1}^{T-1} -\log q_t(\pi^\star_{t+1}).
\]
Hence the supervised objective $\sum_t -\log q_t(\pi^\star_{t+1})$ minimizes a deterministic upper bound on the switch-dependent part of the empirical complexity term. Moreover, if $q_t(k)\ge \varepsilon/K$ for all $t,k$, then
\[
C_{\mathrm{eff}}(\pi^\star)\le \#\mathrm{sw}(\pi^\star)\log(K/\varepsilon),
\]
so the effective switching complexity is always finite.
\end{lemma}

\paragraph{Interpretation.}
Lemma~\ref{lem:ce_switch} is deterministic and pathwise: it does not yet invoke a probabilistic model of oracle paths. It simply states that the $q$-head cross-entropy upper bounds the switch-dependent code length term that appears in the regret guarantee.

\paragraph{Proof sketch.}
$C_{\mathrm{eff}}(\pi^\star)$ is a sub-sum of the full cross-entropy objective restricted to switch times, and $\varepsilon$-uniform mixing keeps every logarithm finite. Full details are given in Appendix~\ref{app:ce_switch}.

\subsection{A pathwise training-aligned regret certificate}
\label{sec:train_aligned}

We now specialize the transition term to the parameterization used by PCGS-TF. Let $p_t\in(0,1)$ denote the switch-probability output, so that $\rho_t=\rho_{\max}p_t$, and let $q_t$ be the restart distribution with $\varepsilon$-uniform mixing. For any path $\pi$, define
\[
s_t(\pi)=\mathbf{1}\{\pi_{t+1}\neq \pi_t\}.
\]

\begin{theorem}[Pathwise regret bound controlled by policy cross-entropies]
\label{thm:train_aligned}
Assume $\rho_{\max}\le \tfrac12$ and $\tilde\ell_{t,k}\in[0,1]$. Run PCGS with constant learning rate $\eta>0$, restart intensity $\rho_t=\rho_{\max}p_t$, and restart distribution $q_t$. Then for any comparator path $\pi$,
\begin{align}
\sum_{t=1}^T \langle w_t,\tilde\ell_t\rangle - \sum_{t=1}^T \tilde\ell_{t,\pi_t}
\;\le\;
\frac{1}{\eta}\Big[
-\log w_1(\pi_1)
&+
\sum_{t=1}^{T-1}s_t(\pi)\Big(\log\tfrac{1}{\rho_{\max}}-\log p_t-\log q_t(\pi_{t+1})\Big)
\notag\\
&\quad+
2\rho_{\max}\sum_{t=1}^{T-1}(1-s_t(\pi))\big(-\log(1-p_t)\big)
\Big]
+\frac{\eta T}{8}.
\label{eq:train_aligned_bound}
\end{align}
\end{theorem}

\paragraph{Interpretation.}
The switch steps are controlled by the same terms used to train PCGS-TF: $-\log q_t(\pi_{t+1})$ for next-expert prediction and $-\log p_t$ for switch propensity. The stay-step term is upper bounded by a function of $-\log(1-p_t)$, which connects the non-switch behavior to the switch-probability head.

\paragraph{Proof sketch.}
Start from Theorem~\ref{thm:pcgs_pathwise}. On switch steps,
\[
-\log A_t(\pi_t\to\pi_{t+1})
=
-\log(\rho_{\max}p_t q_t(\pi_{t+1})).
\]
On stay steps, lower bound $A_t(\pi_t\to\pi_{t+1})$ by $1-\rho_{\max}p_t$ and then use elementary inequalities valid when $\rho_{\max}\le \tfrac12$. Full details appear in Appendix~\ref{app:train_aligned_proof}.
\section{Experiments}
\label{sec:experiments}

Our experiments evaluate the principal Transformer instantiation \textbf{PCGS-TF} under a metric-matched
strictly online protocol. The goal is to test online tracking against a switching oracle, not to optimize
an offline forecasting score. Accordingly, we report a controlled synthetic suite with exact DP
switching-oracle evaluation together with a reproduced household-electricity benchmark that is used as a
reproducibility-oriented external-validity check rather than a broad real-world generalization claim.

\subsection{Research questions (Q1--Q6)}
\label{sec:rq}
We structure the empirical section around six questions aligned with the theory and with the retained
benchmarks:

\begin{enumerate}
\item \textbf{Q1:} Does \textbf{PCGS-TF} reduce dynamic regret against strong baselines (Hedge, Fixed Share, Generalized Share)?
\item \textbf{Q2:} Is the improvement consistent across multiple synthetic non-stationarity mechanisms?
\item \textbf{Q3:} Do the Transformer-controlled restart signals exhibit the mechanism evidence predicted by the theory (e.g., spikes in $\rho_t$ near changes and increased mass on the oracle-next expert)?
\item \textbf{Q4:} Under heavy tails and jumps, does bounded-loss preprocessing together with restart control remain effective?
\item \textbf{Q5:} Does the benefit of learning $q_t$ increase with large expert pools $K$?
\item \textbf{Q6:} Does the performance ordering persist on a reproduced household-electricity benchmark when evaluated under the same strict-online protocol and exact switching-oracle metric?
\end{enumerate}

\subsection{Strict online protocol, tuning fairness, and statistical testing}
\label{sec:fair_protocol}
All methods operate under the same strict online loop: play $w_t$ before observing $\ell_t$; update only after observing $\ell_t$.
Hyperparameters are held fixed across tasks within each benchmark suite (unless explicitly stated), and we report mean$\pm$std across sequences.
For paired comparisons, we use paired $t$-tests and Wilcoxon signed-rank tests (paired by identical sequences/seeds), and report effect sizes (Cohen's $d$) where relevant. This directly answers common reviewer concerns regarding leakage and unfair tuning.

\subsection{Synthetic main suite (Q1--Q2)}
\label{sec:main_suite}
The main simulation suite spans seven families (Switch, Drift, Hetero, HeavyTail, Mix, Predictive, Adversarial) with $T=600$, $K=32$, and $N=20$ sequences per family. Dynamic regret is computed against the DP switching oracle (budget $S$; see caption in the artifact table).

\paragraph{Main quantitative results.}
Table~\ref{tab:main} reports dynamic regret (mean$\pm$std; lower is better). PCGS achieves the best mean performance across all families, with the largest gains in strongly non-stationary regimes (Switch, Adversarial) and consistent improvements in HeavyTail and Hetero.

%

\begin{table}[t]
\centering
\small
\begin{tabular}{lcccc}
\toprule
Family & Hedge & FixedShare & GenShare(heur) & Ours \\
\midrule
Switch & 25.70 $\pm$ 2.46 & 19.58 $\pm$ 0.56 & 19.38 $\pm$ 0.59 & \textbf{18.93 $\pm$ 0.64} \\
Drift & 23.71 $\pm$ 2.39 & 19.59 $\pm$ 0.66 & 19.09 $\pm$ 0.63 & \textbf{18.64 $\pm$ 0.70} \\
Hetero & 28.95 $\pm$ 2.12 & 21.36 $\pm$ 0.43 & 21.16 $\pm$ 0.45 & \textbf{20.73 $\pm$ 0.47} \\
HeavyTail & 27.04 $\pm$ 3.03 & 21.82 $\pm$ 0.56 & 21.69 $\pm$ 0.59 & \textbf{21.17 $\pm$ 0.68} \\
Mix & 23.09 $\pm$ 2.57 & 19.55 $\pm$ 0.44 & 19.52 $\pm$ 0.47 & \textbf{19.39 $\pm$ 0.53} \\
Predictive & 21.04 $\pm$ 1.85 & 16.86 $\pm$ 0.51 & 16.63 $\pm$ 0.51 & \textbf{16.14 $\pm$ 0.59} \\
Adversarial & 31.31 $\pm$ 2.19 & 21.39 $\pm$ 0.41 & 21.19 $\pm$ 0.42 & \textbf{20.64 $\pm$ 0.46} \\
\bottomrule
\end{tabular}
\caption{Main simulation results (Dynamic Regret; mean $\pm$ std over 20 sequences). Lower is better.}
\label{tab:main}
\end{table}

\paragraph{Aggregate visualization.}
Figure~\ref{fig:main_bar} summarizes the same results and highlights the consistency of the gains.

\begin{figure}[t]
  \centering
  \includegraphics[width=0.95\linewidth]{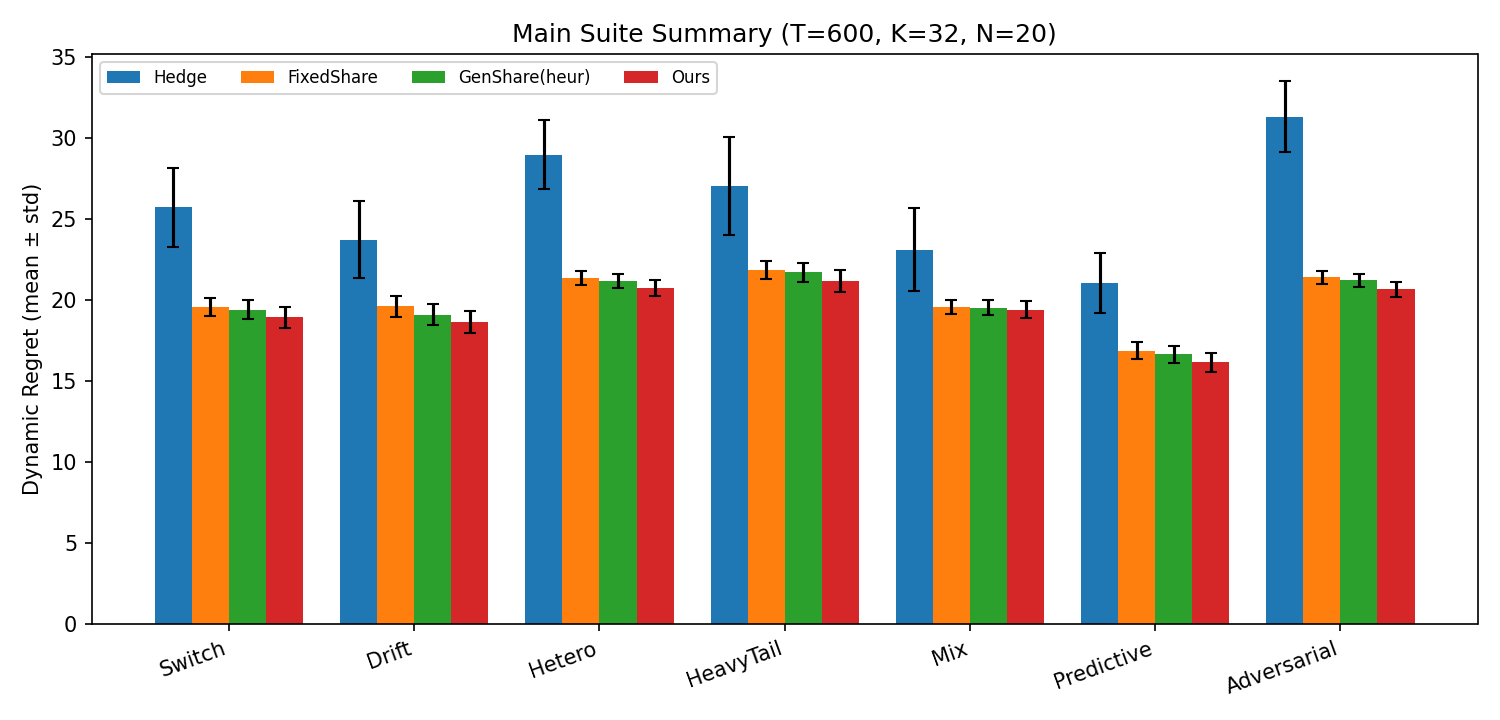}
  \caption{Main suite summary (mean$\pm$std dynamic regret). PCGS improves over FixedShare and GenShare(heur) across all families.}
  \label{fig:main_bar}
\end{figure}

\paragraph{Significance and effect sizes.}
Across families, PCGS exhibits 75--100\% win rates against GenShare(heur) and FixedShare in paired comparisons, with strong effect sizes. For example, versus GenShare(heur) the mean improvement ranges from $0.13$ (Mix) to $0.56$ (Adversarial), with paired test $p$-values typically $\ll 10^{-6}$ and Cohen's $d$ often exceeding $2$ (see Table~\ref{tab:winrates} and appendix CSV logs). Such magnitudes are far beyond marginal tuning artifacts and are consistent with a genuine reduction in switching complexity.

\begin{table}[t]
\centering
\small
\begin{tabular}{lcccc}
\toprule
Family & WinRate vs GenShare & WinRate vs FixedShare & AvgImprov vs GenShare & AvgImprov vs FixedShare \\
\midrule
Adversarial & 100\% & 100\% & 0.56 & 0.75 \\
Drift & 100\% & 100\% & 0.45 & 0.95 \\
HeavyTail & 100\% & 100\% & 0.52 & 0.65 \\
Hetero & 100\% & 100\% & 0.43 & 0.63 \\
Mix & 75\% & 80\% & 0.13 & 0.16 \\
Predictive & 100\% & 100\% & 0.49 & 0.72 \\
Switch & 100\% & 100\% & 0.44 & 0.65 \\
\bottomrule
\end{tabular}
\caption{Win rates and average improvements of Ours over baselines (paired over 20 sequences per family).}
\label{tab:winrates}
\end{table}

\subsection{Mechanism evidence: does the policy behave like the theory predicts? (Q4)}
\label{sec:mechanism}
A mechanism-oriented evaluation should connect the theorem's terms to empirically observable behavior.
Theorem~\ref{thm:pcgs_pathwise} suggests that good performance arises when (i) $\rho_t$ increases near
change points, enabling rapid reallocation, and (ii) $q_t$ assigns high probability to the next-regime
winner, thereby reducing $-\log q_t(\pi^\star_{t+1})$.

Figure~\ref{fig:mechanism_plots} provides representative mechanism plots included in the artifact: (left) $\rho_t$ spikes around switching segments; (right) the restart distribution concentrates mass (and reduces entropy) when the environment becomes predictive of the next winner. These diagnostics make the learned/adaptive control \emph{auditable}.

\begin{figure}[t]
  \centering
  \begin{subfigure}[b]{0.48\linewidth}
    \includegraphics[width=\linewidth]{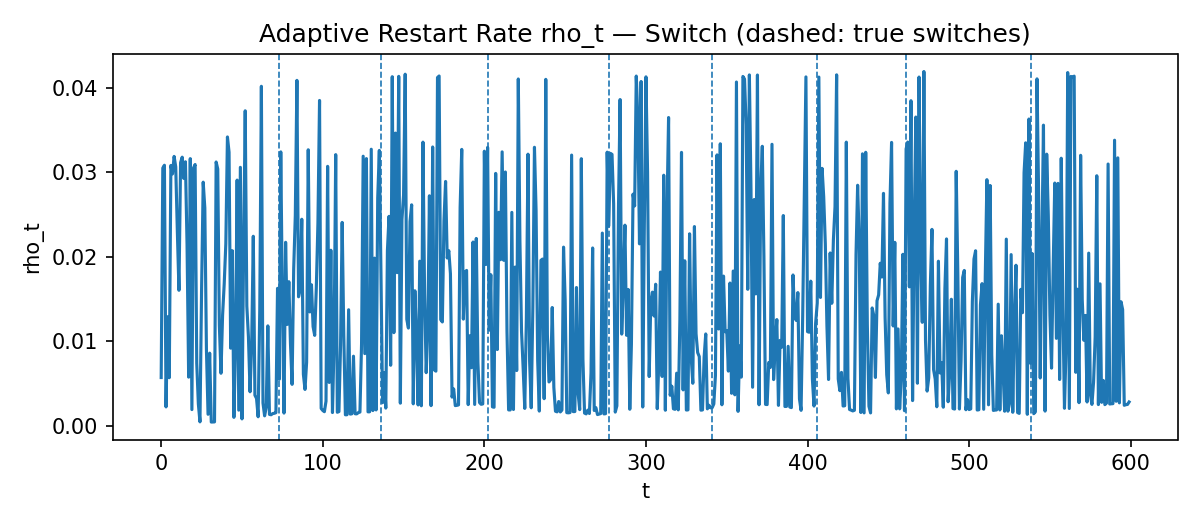}
    \caption{$\rho_t$ trace (Switch family).}
  \end{subfigure}
  \hfill
  \begin{subfigure}[b]{0.48\linewidth}
    \includegraphics[width=\linewidth]{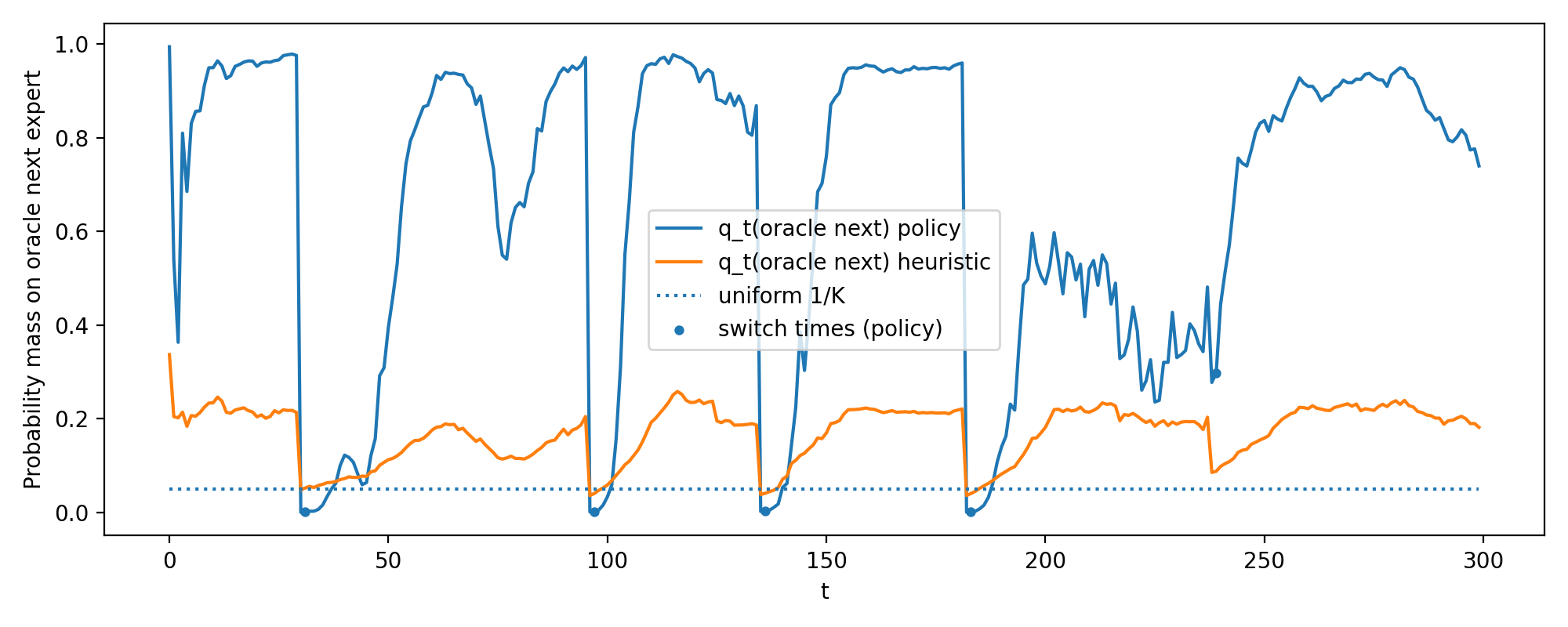}
    \caption{$q_t$ mass on the oracle-next expert (illustrative).}
  \end{subfigure}

  \vspace{0.5em}
  \begin{subfigure}[b]{0.48\linewidth}
    \includegraphics[width=\linewidth]{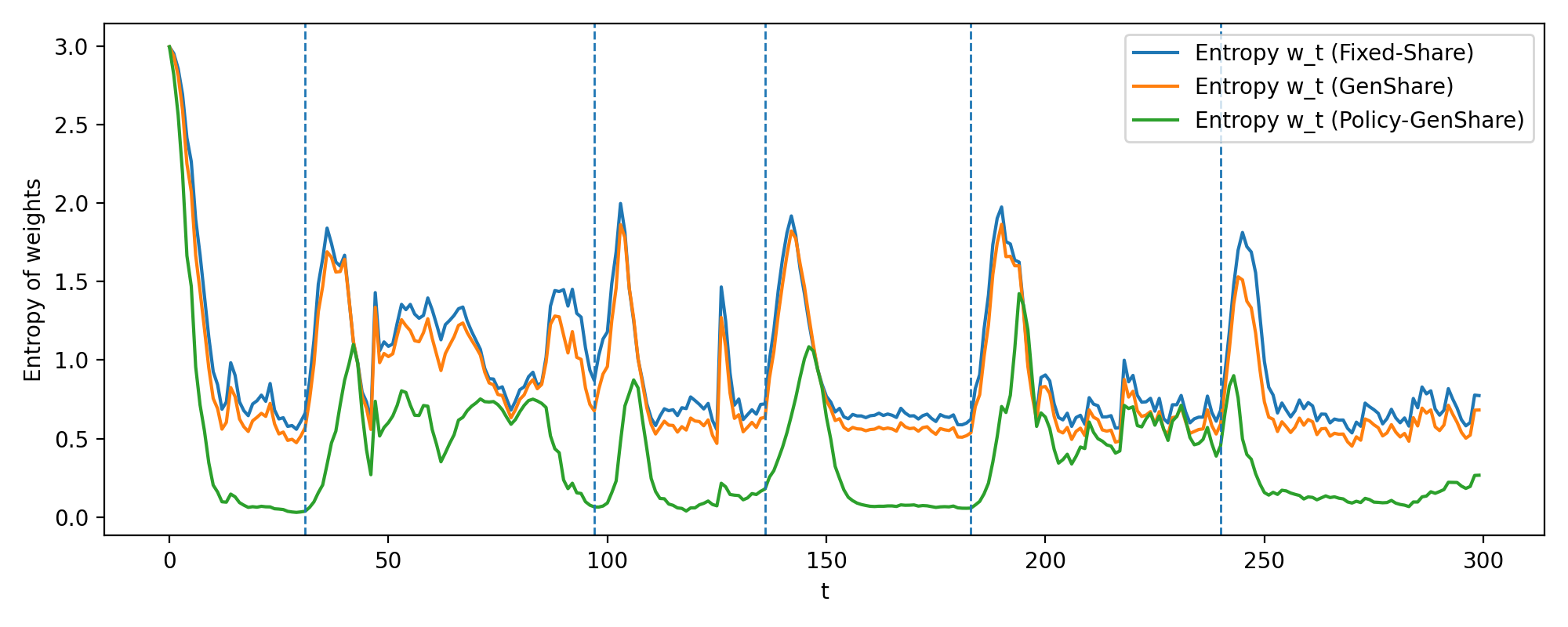}
    \caption{Entropy of $q_t$ (lower $\Rightarrow$ more decisive restarts).}
  \end{subfigure}
  \hfill
  \begin{subfigure}[b]{0.48\linewidth}
    \includegraphics[width=\linewidth]{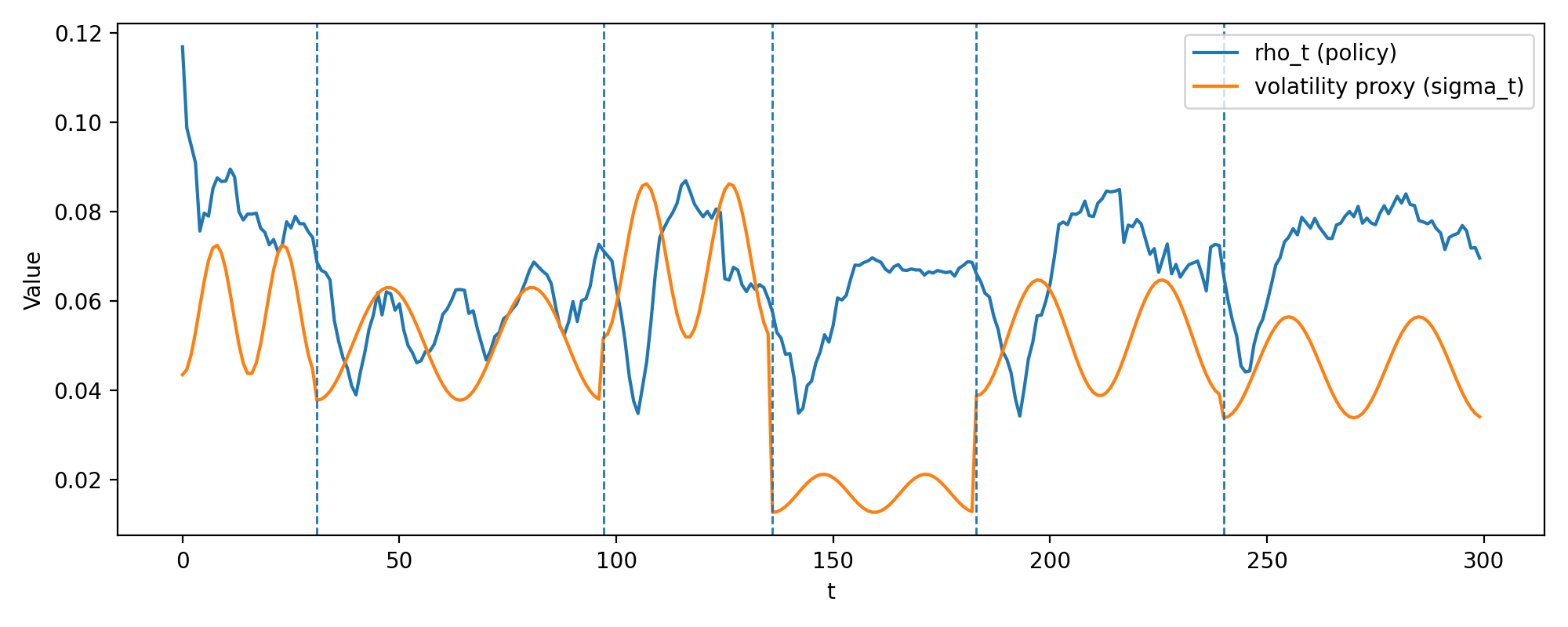}
    \caption{$\rho_t$ vs volatility proxy (illustrative).}
  \end{subfigure}
  \caption{Mechanism evidence connecting theory to behavior. In regimes with abrupt changes, $\rho_t$ increases, enabling rapid reallocation; in predictive regimes, $q_t$ concentrates mass, reducing the effective switching complexity.}
  \label{fig:mechanism_plots}
\end{figure}

\subsection{Robustness under heavy tails and jumps (Q3)}
\label{sec:robustness}
Heavy tails and jumps stress the bounded-loss assumptions and can destabilize multiplicative updates if not handled carefully. We therefore evaluate on a grid over degrees-of-freedom (tail-heaviness) and jump probability. Figure~\ref{fig:heavytail_grid} shows the mean improvement of PCGS over GenShare across the grid; improvements are consistently positive.

\begin{figure}[t]
  \centering
  \includegraphics[width=0.90\linewidth]{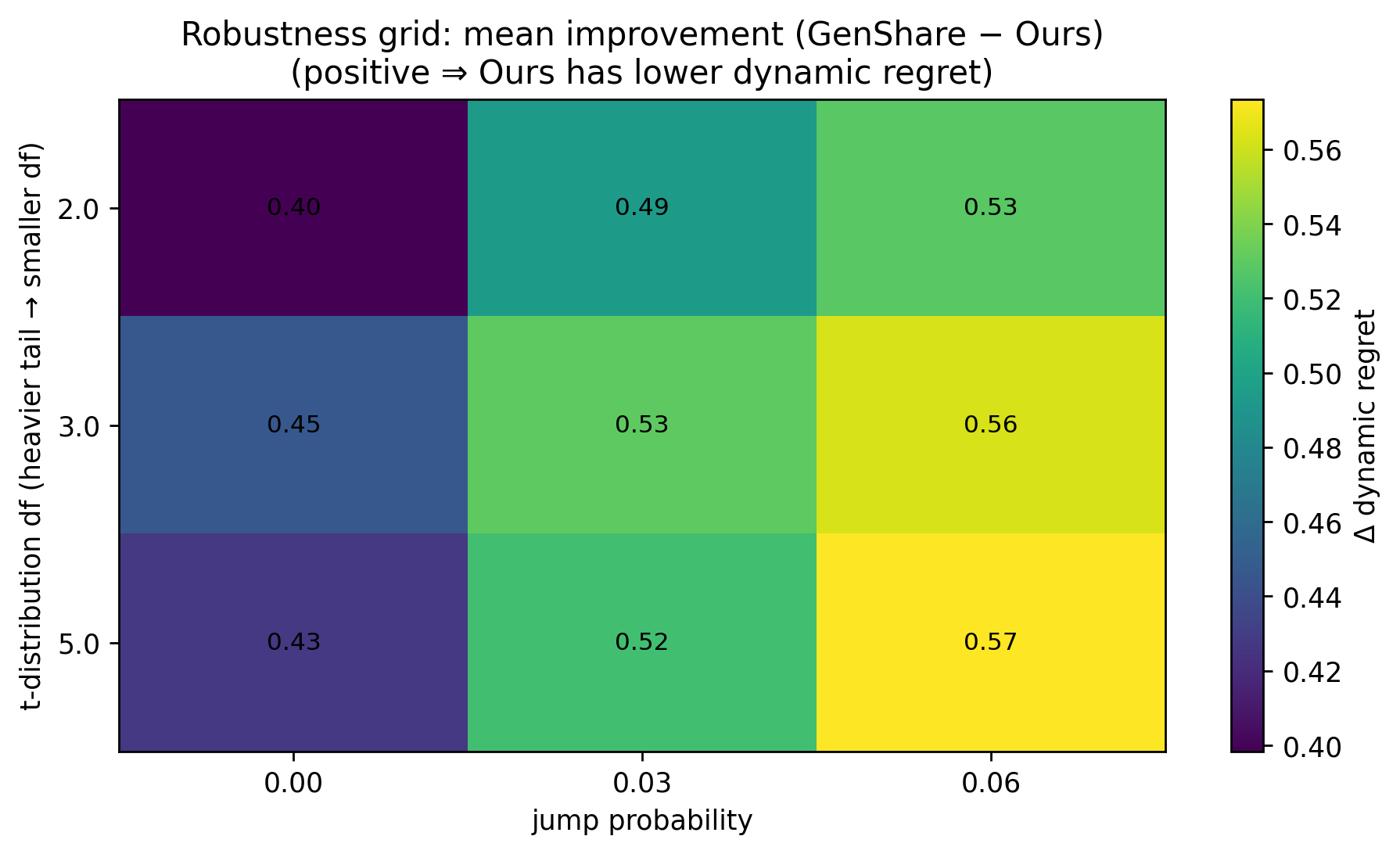}
  \caption{Heavy-tail robustness grid: mean improvement (GenShare DynRegret $-$ Ours DynRegret). Positive values indicate PCGS is better.}
  \label{fig:heavytail_grid}
\end{figure}

For completeness, Table~\ref{tab:heavytail} reports the same grid numerically.

\begin{table}[t]
\centering
\small
\begin{tabular}{lccc}
\toprule
df $\backslash$ jump\_prob & 0.00 & 0.03 & 0.06 \\
\midrule
2.0 & 0.40 & 0.49 & 0.53 \\
3.0 & 0.45 & 0.53 & 0.56 \\
5.0 & 0.43 & 0.52 & 0.57 \\
\bottomrule
\end{tabular}
\caption{Heavy-tail robustness grid: mean improvement (GenShare DynRegret $-$ Ours DynRegret). Positive means Ours is better.}
\label{tab:heavytail}
\end{table}

\subsection{Scaling with large expert pools (Q5)}
\label{sec:scaling}
Because exact DP switching-oracle computation becomes expensive at the largest $K$, this scaling study
uses the known generating path as comparator rather than $L_T^{\mathrm{sw}}(S)$. It should therefore be
read as a targeted stress test of large-$K$ restart allocation, not as a numerically comparable extension
of Table~\ref{tab:main}.

\begin{figure}[t]
  \centering
  \includegraphics[width=0.85\linewidth]{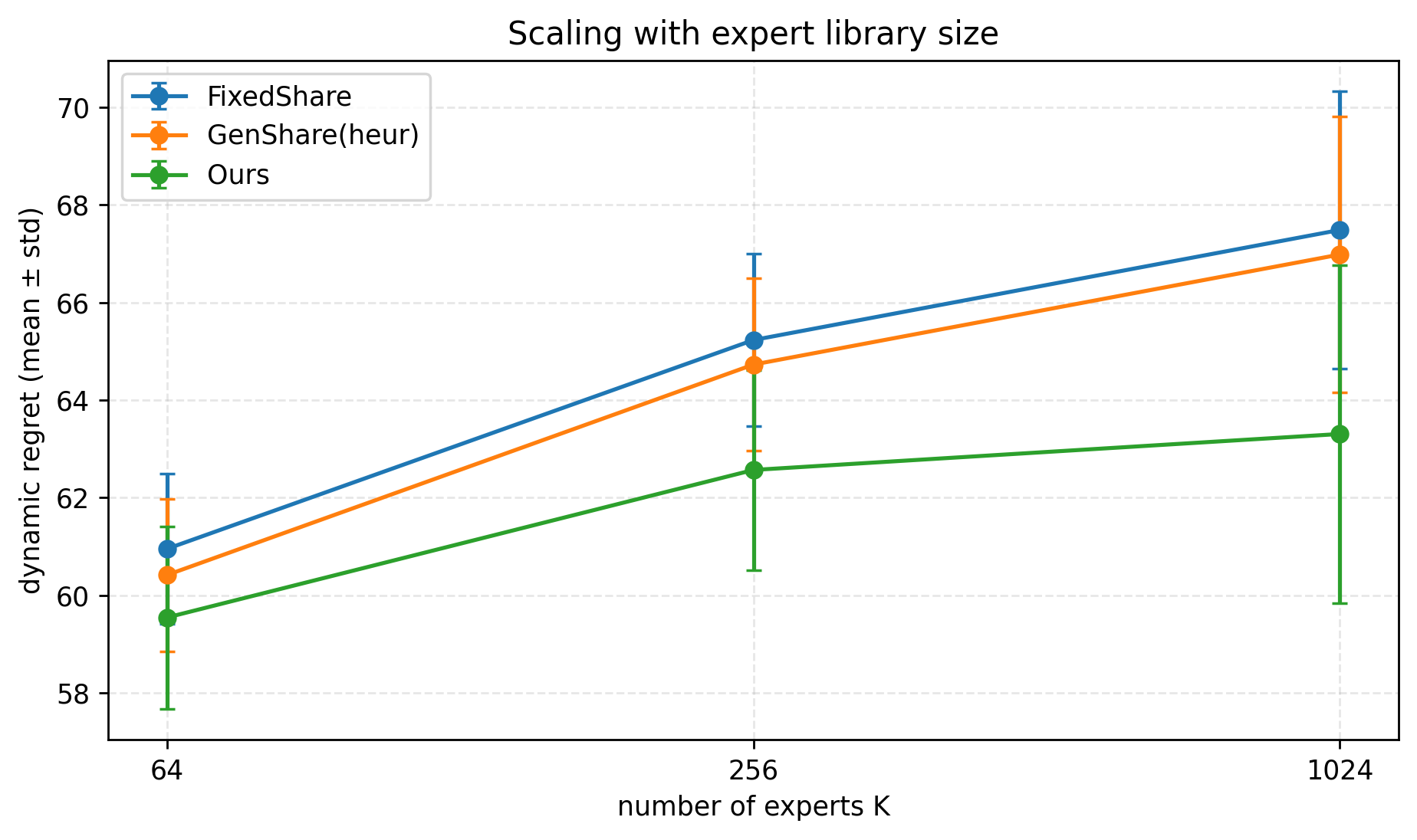}
  \caption{Scaling with $K\in\{64,256,1024\}$. The advantage of policy-controlled restarts increases with larger expert libraries.}
  \label{fig:scalingK}
\end{figure}

\begin{table}[t]
\centering
\small
\begin{tabular}{rccc}
\toprule
K & FixedShare & GenShare(heur) & Ours \\
\midrule
64 & 60.95 $\pm$ 1.54 & 60.41 $\pm$ 1.56 & \textbf{59.54 $\pm$ 1.87} \\
256 & 65.23 $\pm$ 1.76 & 64.73 $\pm$ 1.77 & \textbf{62.57 $\pm$ 2.05} \\
1024 & 67.49 $\pm$ 2.84 & 66.98 $\pm$ 2.83 & \textbf{63.30 $\pm$ 3.47} \\
\bottomrule
\end{tabular}
\caption{Scaling with large expert pools. To avoid the cost of exact DP switching-oracle computation at
$K=1024$, this stress test uses the ground-truth piecewise-optimal generating path as comparator rather
than $L_T^{\mathrm{sw}}(S)$. The values therefore illustrate the large-$K$ trend but are not directly
numerically comparable to Table~\ref{tab:main}.}
\label{tab:scaling}
\end{table}
\subsection{Real-data reproducibility benchmark: household electricity consumption}
\label{sec:real_electricity}

To complement the controlled synthetic study, we include a real-data benchmark designed as a \emph{reproducibility-oriented external validity check}. Concretely, we reproduce a previously reported electricity experiment from a frozen snapshot pipeline and re-evaluate all methods under the same strict online protocol and the same exact DP switching-oracle metric. This benchmark is intentionally narrower than the synthetic suite, but it is valuable because it tests whether the ordering observed in simulation survives on a real, non-stationary stream with seasonality, local spikes, and recurring regime structure.

\paragraph{Dataset and protocol.}
We use the UCI household electricity consumption dataset. Following the reproduced pipeline, we remove missing values, downsample to 10-minute resolution, keep the final 20{,}000 observations, apply a 70/30 temporal split, use a rolling window of length 1440, and set \texttt{burn\_in}=48. The expert library contains lag-based predictors, moving averages, EWMA experts, and online linear regression. Performance is reported as normalized dynamic regret (dynreg$/T$) against the exact DP switching oracle under switch budgets $S\in\{5,10,20\}$. All methods are evaluated on the same loss matrices and under the same strictly online decision/update timing.

\paragraph{Main result.}
Table~\ref{tab:real_electricity} shows that \textbf{Ours is best for every switch budget}. Averaged over $S\in\{5,10,20\}$, Ours reduces dynreg$/T$ from $0.01000$ to $0.00817$ relative to GenShare(heur), an average relative reduction of $18.8\%$. The improvement over FixedShare is substantially larger ($56.7\%$ on average), indicating that adaptive restart intensity and restart destination remain useful on real non-stationary data rather than only on synthetic generators.

\paragraph{Interpretation.}
Two aspects are especially encouraging from a reviewer perspective. First, the ranking is \emph{stable across all switch budgets}, which argues against a budget-specific tuning artifact. Second, the gain persists as $S$ increases from 5 to 20: although the relative gap to GenShare(heur) decreases mildly, Ours remains uniformly best, suggesting that the policy is improving tracking behavior rather than merely exploiting a single sharp-change regime. Since this reproduced benchmark currently reports single-run values from a frozen snapshot pipeline, we present it as a reproducibility-oriented real-data validation rather than as a standalone statistical study.

\begin{table}[t]
\centering
\small
\begin{tabular}{rccccc}
\toprule
$S$ & Hedge & FixedShare & GenShare(heur) & Ours & Rel.\ Gain vs GenShare \\
\midrule
5  & 0.01756 & 0.01684 & 0.00813 & \textbf{0.00630} & 22.5\% \\
10 & 0.01928 & 0.01856 & 0.00985 & \textbf{0.00802} & 18.6\% \\
20 & 0.02144 & 0.02072 & 0.01202 & \textbf{0.01018} & 15.3\% \\
\midrule
Avg. & 0.01943 & 0.01871 & 0.01000 & \textbf{0.00817} & 18.8\% \\
\bottomrule
\end{tabular}
\caption{Real-data benchmark on household electricity consumption. Values are normalized dynamic regret (dynreg$/T$) against the exact DP switching oracle; lower is better. Ours achieves the best value for every switch budget $S\in\{5,10,20\}$.}
\label{tab:real_electricity}
\end{table}

\begin{figure}[t]
  \centering
  \begin{subfigure}[b]{0.32\linewidth}
    \includegraphics[width=\linewidth]{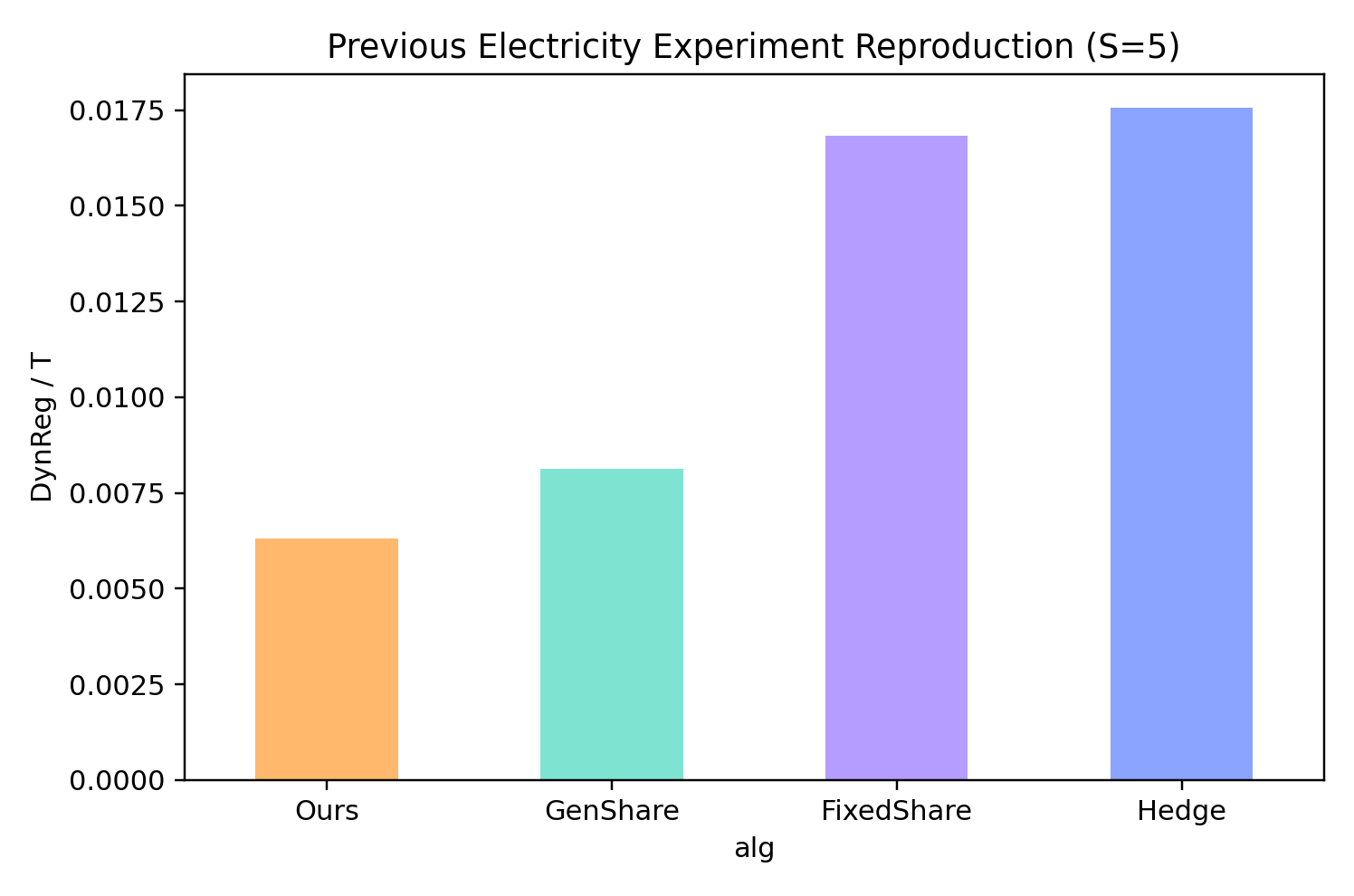}
    \caption{$S=5$}
  \end{subfigure}
  \hfill
  \begin{subfigure}[b]{0.32\linewidth}
    \includegraphics[width=\linewidth]{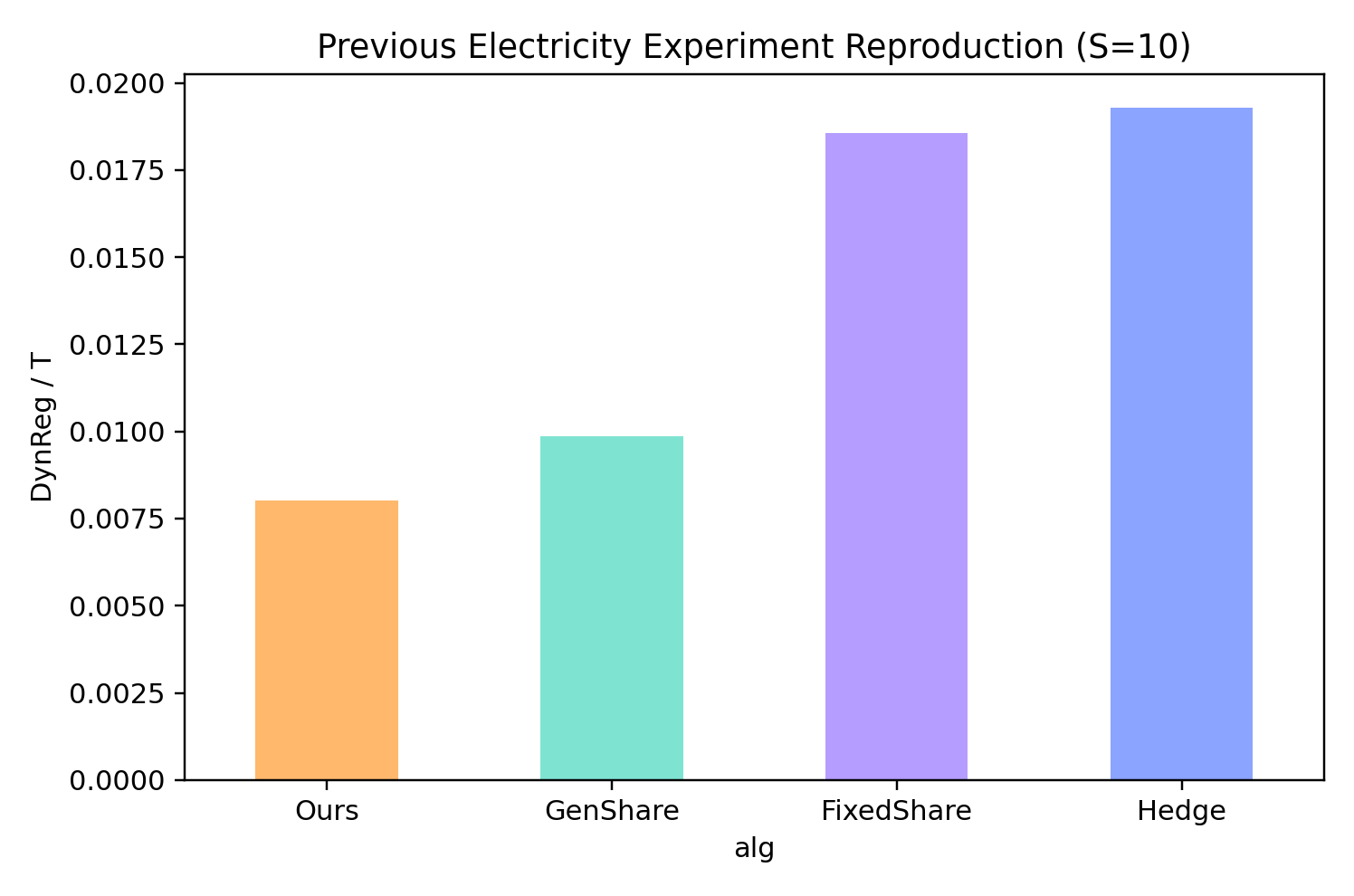}
    \caption{$S=10$}
  \end{subfigure}
  \hfill
  \begin{subfigure}[b]{0.32\linewidth}
    \includegraphics[width=\linewidth]{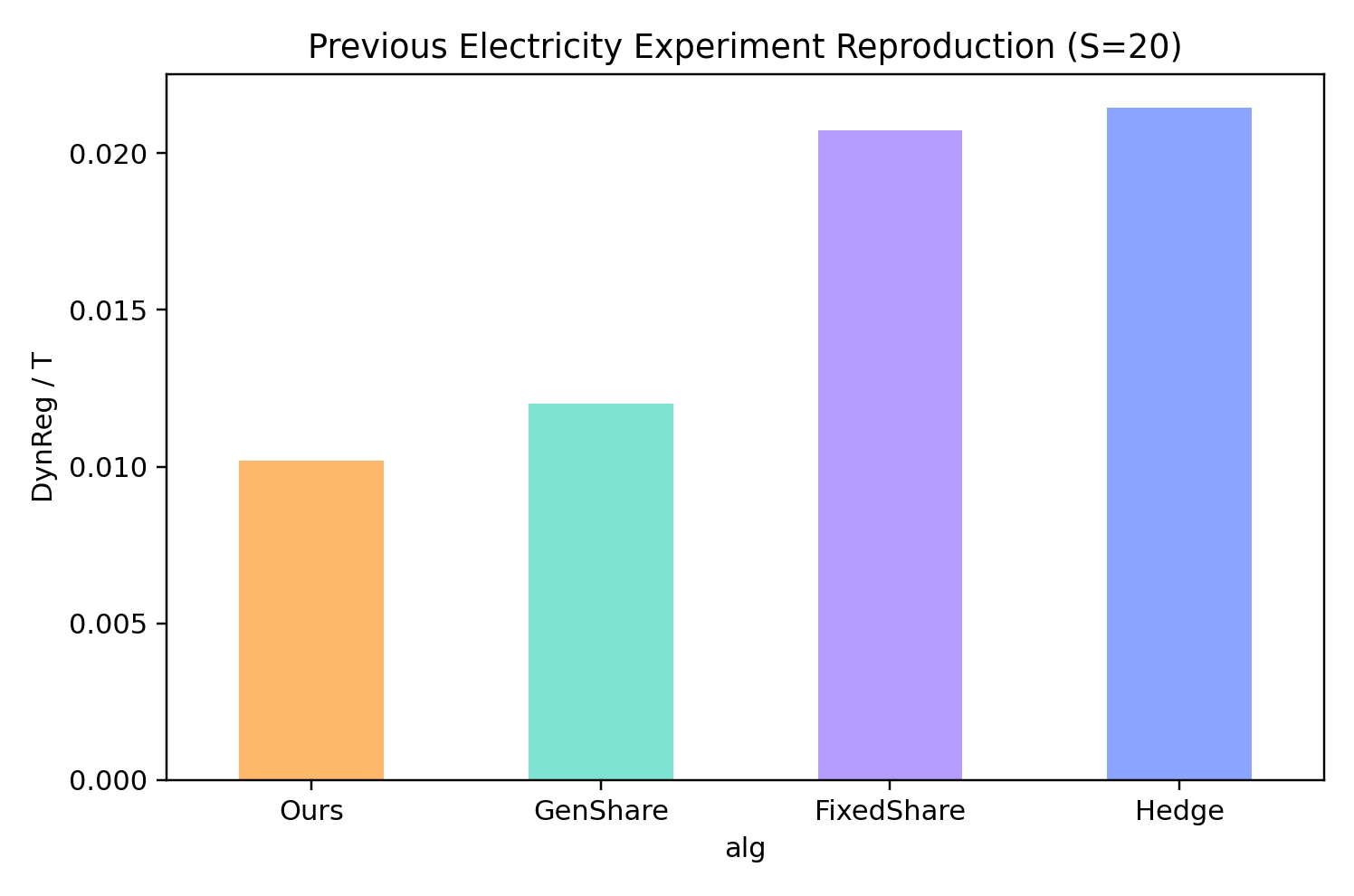}
    \caption{$S=20$}
  \end{subfigure}
  \caption{Reproduced real-data electricity benchmark. Across all switch budgets, Ours achieves the lowest dynreg$/T$, and the performance ordering is stable as the oracle switch budget increases.}
  \label{fig:real_electricity}
\end{figure}

\subsection{Ablations (brief)}
\label{sec:ablations}
Ablations isolate which elements matter most (restart adaptivity, restart destination learning, clipping).
On Mix and Predictive benchmarks, the artifact tables show that removing $\rho$-adaptation or forcing
uniform $q_t$ degrades performance, consistent with the theory that both $\rho_t$ and $q_t$ control
switching complexity. See the released artifact tables \texttt{table\_ablation\_mix.tex} and
\texttt{table\_ablation\_predictive.tex} for the full numerical breakdown.

\section{Discussion, Limitations, and Broader Impact}
\label{sec:discussion}

\subsection{What PCGS-TF says about Transformers in online learning}
A recurring debate is whether Transformers help because they are better direct predictors or because they
can implement more general algorithmic behavior. PCGS-TF provides one concrete answer in a strictly online
setting: a Transformer can act as a strictly causal update controller for a theory-grounded expert-tracking
algorithm. In our setting, the Transformer does not emit the round-$t$ prediction. Instead, it reads the
realized online history and controls how the generalized-share backbone updates from $w_t$ to $w_{t+1}$.
This separation between backbone and controller makes the method auditable through $(\eta_t,\rho_t,q_t)$
traces while retaining explicit regret guarantees for the underlying online learner.

\subsection{Limitations}
\paragraph{Oracle supervision and DP complexity.}
Training uses switching-oracle trajectories computed by DP with cost $O(TKS)$. This is feasible for our benchmarks but can be expensive at large scale. Approximate oracles (beam search, segment-wise DP, or learned surrogates) are promising directions.

\paragraph{Policy class and optimization.}
Although our theory is agnostic to the policy model, the empirical behavior of $q_t$ and $\rho_t$ depends on the representation and training regime. For example, a minimal Transformer with frozen attention may underfit some tasks; fully trained attention may perform better but adds complexity. Studying the trade-off between policy expressivity and online stability is an important research topic.

\paragraph{Bounded-loss surrogate vs raw loss.}
Clipping enforces theoretical assumptions and robustness, but it changes the objective. In some domains, one may want to bound the regret in terms of unclipped losses or use robust loss transformations that preserve more information (e.g., Catoni-type constructions \cite{catoni2012challenging}). This is compatible with PCGS and mainly affects the choice of $\tilde\ell_{t,k}$ and token features.

\paragraph{From forecasting benchmarks to deployment.}
The synthetic suite and the reproduced household-electricity benchmark evaluate strictly online tracking of
expert forecasters. They do not by themselves establish deployment performance under domain-specific
operational constraints such as missing-data interventions, action costs, actuation delays, or feedback
effects. We therefore interpret the empirical results as evidence about online adaptation, not as a
turnkey deployment claim.
\subsection{Broader impact and responsible use}
PCGS is a general online-learning framework that can be applied in high-stakes settings (finance, operations, safety monitoring). Responsible use requires strict separation between training and evaluation distributions, transparent reporting of online protocols, and careful auditing of failure modes (e.g., overly aggressive restarts under transient noise). The interpretability of $(\eta_t,\rho_t,q_t)$ traces can support monitoring, debugging, and governance.


\appendix
\section{Full Proofs and Algorithmic Details}
\label{app:proofs}

\subsection{A. Probability space, filtration, and strict online measurability}
\label{app:filtration}

We formalize the strictly online protocol in a way that makes ``no look-ahead'' precise and eliminates common ambiguity.

\paragraph{Basic objects.}
Let $(\Omega,\mathcal{F},\mathbb{P})$ be an underlying probability space. The online interaction unfolds over rounds $t=1,\dots,T$ with a filtration $\{\mathcal{F}_t\}_{t=0}^T$.
We interpret $\mathcal{F}_t$ as the $\sigma$-algebra containing all information revealed \emph{up to and including} the loss vector $\ell_t$ at round $t$.

There are $K$ experts. The environment reveals at each time $t$ a loss vector
\[
\ell_t = (\ell_{t,1},\dots,\ell_{t,K}) \in [0,1]^K,
\]
which is assumed $\mathcal{F}_t$-measurable. (No stochastic assumptions such as independence are required for the regret analysis; the results are pathwise.)

The learner chooses a distribution $w_t\in\Delta_K$ over experts. The key causal constraint is:

\begin{assumption}[Strictly online decision rule]
\label{ass:strict_online}
For each $t$, the played distribution $w_t$ is $\mathcal{F}_{t-1}$-measurable.
\end{assumption}

In words: $w_t$ must be committed \emph{before} seeing $\ell_t$.

\paragraph{Update parameters and timing.}
PCGS uses controls $(\eta_t,\rho_t,q_t)$ to update $w_t\mapsto w_{t+1}$. In our implementation, these controls are produced \emph{after} observing $\ell_t$, hence they may depend on $\mathcal{F}_t$.

We record this explicitly:

\begin{assumption}[Causal update parameters]
\label{ass:causal_params}
For each $t\in\{1,\dots,T-1\}$, the update parameters $(\eta_t,\rho_t,q_t)$ (equivalently the transition kernel $A_t$ defined in \eqref{eq:At_def} below) are $\mathcal{F}_t$-measurable.
\end{assumption}

This is consistent with strict online learning: parameters chosen after $\ell_t$ can affect only $w_{t+1}$ and beyond.

\paragraph{A simple but important lemma.}
The next lemma shows that if we initialize causally and update using only current losses, then strict online measurability holds for all rounds.

\begin{lemma}[PCGS is strictly online by construction]
\label{lem:strict_online_induction}
Let $\Delta_K \triangleq \{w\in\mathbb{R}_+^K:\sum_{k=1}^K w(k)=1\}$ be the simplex of expert mixtures. Suppose that:
\begin{enumerate}
    \item the initial mixture $w_1$ is deterministic, or more generally $\mathcal{F}_0$-measurable;
    \item for each $t\in\{1,\dots,T-1\}$, the loss vector $\ell_t\in[0,1]^K$ is $\mathcal{F}_t$-measurable;
    \item for each $t\in\{1,\dots,T-1\}$, the control variables $(\eta_t,\rho_t,q_t)$ are $\mathcal{F}_t$-measurable;
    \item for each $t\in\{1,\dots,T-1\}$, there exists a measurable update map
    \[
    \mathsf{Update}_t:
    \Delta_K \times [0,1]^K \times \mathbb{R}_+ \times [0,1)\times \Delta_K
    \longrightarrow
    \Delta_K
    \]
    such that
    \[
    w_{t+1}
    =
    \mathsf{Update}_t\!\big(w_t,\ell_t,\eta_t,\rho_t,q_t\big).
    \]
\end{enumerate}
Then, for every $t\in\{1,\dots,T\}$, the played mixture $w_t$ is $\mathcal{F}_{t-1}$-measurable. In particular, the sequence $\{w_t\}_{t=1}^T$ is strictly online adapted in the sense required by the protocol.
\end{lemma}

\begin{proof}
We argue by induction on $t$. The case $t=1$ is immediate from the assumption that $w_1$ is deterministic or, more generally, $\mathcal{F}_0$-measurable. Since $\mathcal{F}_0=\mathcal{F}_{1-1}$, the first-round mixture is measurable with respect to the information available before round $1$ begins.

Now fix $t\in\{1,\dots,T-1\}$ and assume that $w_t$ is $\mathcal{F}_{t-1}$-measurable. Because $\{\mathcal{F}_t\}_{t=0}^T$ is a filtration, we have $\mathcal{F}_{t-1}\subseteq \mathcal{F}_t$, and therefore $w_t$ is also $\mathcal{F}_t$-measurable. By assumption, the current loss vector $\ell_t$ and the controls $(\eta_t,\rho_t,q_t)$ are $\mathcal{F}_t$-measurable as well. It follows that the joint random element
\[
\omega \longmapsto \big(w_t(\omega),\ell_t(\omega),\eta_t(\omega),\rho_t(\omega),q_t(\omega)\big)
\]
from $(\Omega,\mathcal{F}_t)$ into
\[
\Delta_K \times [0,1]^K \times \mathbb{R}_+ \times [0,1)\times \Delta_K
\]
is $\mathcal{F}_t$-measurable.

Since $\mathsf{Update}_t$ is measurable as a map on this product space, the composition
\[
w_{t+1}
=
\mathsf{Update}_t(w_t,\ell_t,\eta_t,\rho_t,q_t)
\]
is $\mathcal{F}_t$-measurable. Finally, because $\mathcal{F}_t=\mathcal{F}_{(t+1)-1}$, we conclude that $w_{t+1}$ is measurable with respect to the information available before round $t+1$ begins. This is exactly the strict online requirement at time $t+1$.

Thus the implication
\[
w_t \text{ is } \mathcal{F}_{t-1}\text{-measurable}
\quad\Longrightarrow\quad
w_{t+1} \text{ is } \mathcal{F}_t\text{-measurable}
=
\mathcal{F}_{(t+1)-1}\text{-measurable}
\]
holds for every $t\in\{1,\dots,T-1\}$. By induction, $w_t$ is $\mathcal{F}_{t-1}$-measurable for all $t\in\{1,\dots,T\}$. Hence \textsc{PCGS} is strictly online by construction.
\end{proof}

\subsection{B. Switching oracle via dynamic programming: correctness and complexity}
\label{app:dp_oracle}

We provide a fully rigorous derivation of the DP used to compute the switching oracle
\[
L_T^{\mathrm{sw}}(S) = \min_{\pi\in\Pi_S} \sum_{t=1}^T \ell_{t,\pi_t},
\qquad
\Pi_S=\{\pi\in[K]^T:\#\mathrm{sw}(\pi)\le S\}.
\]

\subsubsection{B.1 DP state and recurrence}
For $t\in\{1,\dots,T\}$, $s\in\{0,\dots,S\}$, and $k\in[K]$, define
\begin{equation}
\mathrm{dp}[t,s,k]
\;\triangleq\;
\min_{\substack{\pi_{1:t}\in[K]^t:\\ \pi_t=k,\; \#\mathrm{sw}(\pi_{1:t})\le s}}
\;\sum_{\tau=1}^t \ell_{\tau,\pi_\tau}.
\label{eq:dp_def}
\end{equation}
We use the convention $\mathrm{dp}[t,s,k]=+\infty$ if the constraint set is empty (e.g., $s<0$).

\paragraph{Initialization.}
At $t=1$, no switch is possible, hence for all $s\ge 0$,
\begin{equation}
\mathrm{dp}[1,s,k] = \ell_{1,k}.
\label{eq:dp_init}
\end{equation}

\paragraph{Recurrence.}
For $t\ge 2$ and $s\ge 0$,
\begin{equation}
\mathrm{dp}[t,s,k]
=
\ell_{t,k}
+
\min\Big\{
\mathrm{dp}[t-1,s,k],\;
\min_{j\neq k}\mathrm{dp}[t-1,s-1,j]
\Big\}.
\label{eq:dp_rec}
\end{equation}

\begin{lemma}[Correctness of the DP recurrence]
\label{lem:dp_correctness}
For every $t\in\{2,\dots,T\}$, every $s\in\{0,\dots,S\}$, and every $k\in[K]$, the dynamic-programming recurrence
\begin{equation}
\mathrm{dp}[t,s,k]
=
\ell_{t,k}
+
\min\Big\{
\mathrm{dp}[t-1,s,k],\;
\min_{j\neq k}\mathrm{dp}[t-1,s-1,j]
\Big\}
\label{eq:dp_rec_recalled}
\end{equation}
correctly computes the value
\[
\mathrm{dp}[t,s,k]
\;\triangleq\;
\min_{\substack{\pi_{1:t}\in[K]^t:\\ \pi_t=k,\; \#\mathrm{sw}(\pi_{1:t})\le s}}
\sum_{\tau=1}^t \ell_{\tau,\pi_\tau}.
\]
\end{lemma}

\begin{proof}
Fix $t\ge 2$, $s\ge 0$, and $k\in[K]$. We show that the right-hand side of \eqref{eq:dp_rec_recalled} is exactly the minimum cost among all length-$t$ paths that end at expert $k$ and use at most $s$ switches.

Let $\pi_{1:t}$ be any admissible path appearing in the definition of $\mathrm{dp}[t,s,k]$, so that $\pi_t=k$ and $\#\mathrm{sw}(\pi_{1:t})\le s$. Write $j=\pi_{t-1}$ for the expert used at time $t-1$. Since the terminal expert at time $t$ is fixed to be $k$, the total loss of $\pi_{1:t}$ decomposes as
\[
\sum_{\tau=1}^t \ell_{\tau,\pi_\tau}
=
\sum_{\tau=1}^{t-1}\ell_{\tau,\pi_\tau}+\ell_{t,k}.
\]
Thus the problem reduces to understanding which admissible prefix costs may appear in the first term.

If $j=k$, then no new switch is incurred between times $t-1$ and $t$, and therefore the prefix $\pi_{1:t-1}$ still satisfies
\[
\#\mathrm{sw}(\pi_{1:t-1})\le s,
\qquad
\pi_{t-1}=k.
\]
By the definition of $\mathrm{dp}[t-1,s,k]$, this implies
\[
\sum_{\tau=1}^{t-1}\ell_{\tau,\pi_\tau}\ge \mathrm{dp}[t-1,s,k].
\]
Consequently,
\[
\sum_{\tau=1}^t \ell_{\tau,\pi_\tau}
\ge
\mathrm{dp}[t-1,s,k]+\ell_{t,k}.
\]

If instead $j\neq k$, then the transition from time $t-1$ to time $t$ contributes exactly one switch, so removing the terminal point decreases the switch count by one:
\[
\#\mathrm{sw}(\pi_{1:t-1})
=
\#\mathrm{sw}(\pi_{1:t})-1
\le s-1.
\]
Moreover, the prefix ends at expert $j\neq k$. Hence the definition of the dynamic program gives
\[
\sum_{\tau=1}^{t-1}\ell_{\tau,\pi_\tau}\ge \mathrm{dp}[t-1,s-1,j]
\ge \min_{j'\neq k}\mathrm{dp}[t-1,s-1,j'].
\]
It follows that
\[
\sum_{\tau=1}^t \ell_{\tau,\pi_\tau}
\ge
\ell_{t,k}
+
\min_{j'\neq k}\mathrm{dp}[t-1,s-1,j'].
\]

Since every admissible path $\pi_{1:t}$ ending at $k$ must have some predecessor $j=\pi_{t-1}$, the two inequalities above imply
\[
\sum_{\tau=1}^t \ell_{\tau,\pi_\tau}
\ge
\ell_{t,k}
+
\min\Big\{
\mathrm{dp}[t-1,s,k],\;
\min_{j\neq k}\mathrm{dp}[t-1,s-1,j]
\Big\}.
\]
Taking the minimum over all admissible $\pi_{1:t}$ yields
\begin{equation}
\mathrm{dp}[t,s,k]
\ge
\ell_{t,k}
+
\min\Big\{
\mathrm{dp}[t-1,s,k],\;
\min_{j\neq k}\mathrm{dp}[t-1,s-1,j]
\Big\}.
\label{eq:dp_lower_bound}
\end{equation}

It remains to prove the reverse inequality. For this, it suffices to show that each admissible predecessor term appearing on the right-hand side can be extended to a valid path of length $t$ ending at $k$.

Consider first any path achieving $\mathrm{dp}[t-1,s,k]$. Such a path exists whenever the state is feasible; if the state is infeasible, then by convention the corresponding value is $+\infty$ and there is nothing to prove. Appending expert $k$ at time $t$ produces a path of length $t$ that still ends at $k$ and incurs no additional switch on the last transition. Therefore the resulting path is admissible for the state $(t,s,k)$ and has total cost
\[
\mathrm{dp}[t-1,s,k]+\ell_{t,k}.
\]

Next, fix any $j\neq k$ for which $\mathrm{dp}[t-1,s-1,j]$ is finite. Take a path achieving this value. Appending expert $k$ at time $t$ produces a path of length $t$ ending at $k$ and increases the number of switches by exactly one, because the final transition is from $j$ to $k$ with $j\neq k$. Hence the resulting path has at most $s$ switches and is admissible for $(t,s,k)$. Its total cost is
\[
\mathrm{dp}[t-1,s-1,j]+\ell_{t,k}.
\]
Since this is true for every $j\neq k$,
\[
\mathrm{dp}[t,s,k]
\le
\ell_{t,k}
+
\min_{j\neq k}\mathrm{dp}[t-1,s-1,j].
\]
Combining this with the previous extension argument from the stay-at-$k$ predecessor yields
\begin{equation}
\mathrm{dp}[t,s,k]
\le
\ell_{t,k}
+
\min\Big\{
\mathrm{dp}[t-1,s,k],\;
\min_{j\neq k}\mathrm{dp}[t-1,s-1,j]
\Big\}.
\label{eq:dp_upper_bound}
\end{equation}

The lower bound \eqref{eq:dp_lower_bound} and the upper bound \eqref{eq:dp_upper_bound} coincide, and therefore \eqref{eq:dp_rec_recalled} holds. This proves that the recurrence computes the correct optimal value for every state $(t,s,k)$.
\end{proof}

\subsubsection{B.2 Oracle value and path recovery}
The switching oracle value is
\begin{equation}
L_T^{\mathrm{sw}}(S)=\min_{k\in[K]} \mathrm{dp}[T,S,k].
\label{eq:dp_value}
\end{equation}
To recover an optimizer $\pi^\star$, store backpointers at each state $(t,s,k)$ indicating whether the minimum in \eqref{eq:dp_rec} was achieved by staying ($k\leftarrow k$) or switching ($k\leftarrow j$), together with the predecessor switch budget ($s$ or $s-1$). Backtracking from the minimizing terminal state yields $\pi^\star_{1:T}$.

\subsubsection{B.3 Complexity and the best/second-best trick}
A naive implementation of \eqref{eq:dp_rec} computes $\min_{j\neq k}$ in $O(K)$ for each $k$, hence $O(TSK^2)$ overall. We can reduce this to $O(TSK)$.

\begin{lemma}[Best/second-best trick]
\label{lem:best_second_best}
Fix $(t,s)$ with $t\ge 2$ and $s\ge 1$. Define
\[
a_j \;\triangleq\; \mathrm{dp}[t-1,s-1,j], \qquad j\in[K].
\]
Let
\[
m_1 \;\triangleq\; \min_{j\in[K]} a_j,
\qquad
j_1 \in \arg\min_{j\in[K]} a_j,
\]
where $j_1$ is chosen according to any deterministic tie-breaking rule. Define also
\[
m_2 \;\triangleq\; \min_{j\in[K]\setminus\{j_1\}} a_j,
\]
with the convention that ties are handled consistently under the same rule. Then, for every $k\in[K]$,
\[
\min_{j\neq k}\mathrm{dp}[t-1,s-1,j]
=
\min_{j\in[K]\setminus\{k\}} a_j
=
\begin{cases}
m_1, & k\neq j_1,\\[4pt]
m_2, & k=j_1.
\end{cases}
\]
\end{lemma}

\begin{proof}
Fix $k\in[K]$. By definition,
\[
\min_{j\neq k}\mathrm{dp}[t-1,s-1,j]
=
\min_{j\in[K]\setminus\{k\}} a_j.
\]
Thus the problem is to determine the minimum of the collection $\{a_j\}_{j\in[K]}$ after excluding the single index $k$.

Since $j_1$ is chosen from $\arg\min_{j\in[K]} a_j$, one has
\[
a_{j_1}=m_1
\qquad\text{and}\qquad
a_j\ge m_1 \quad \text{for all } j\in[K].
\]
Moreover, by the definition of $m_2$,
\[
m_2=\min_{j\in[K]\setminus\{j_1\}} a_j,
\]
so $m_2$ is exactly the smallest value attained among all indices other than $j_1$.

Suppose first that $k\neq j_1$. Then the index $j_1$ remains feasible in the constrained minimization over $[K]\setminus\{k\}$. Therefore,
\[
\min_{j\in[K]\setminus\{k\}} a_j
\le a_{j_1}=m_1.
\]
On the other hand, since $m_1$ is the minimum over the entire index set $[K]$, restricting the minimization to the smaller feasible set $[K]\setminus\{k\}$ cannot produce a value strictly below $m_1$. Hence
\[
\min_{j\in[K]\setminus\{k\}} a_j \ge m_1.
\]
Combining the two inequalities yields
\[
\min_{j\in[K]\setminus\{k\}} a_j = m_1
\qquad\text{whenever } k\neq j_1.
\]

Now suppose that $k=j_1$. Then the unique designated minimizer $j_1$ selected by the tie-breaking rule is removed from the feasible set, and the constrained minimum becomes
\[
\min_{j\in[K]\setminus\{j_1\}} a_j.
\]
By definition, this quantity is precisely $m_2$. That is,
\[
\min_{j\in[K]\setminus\{j_1\}} a_j = m_2.
\]

Since the two possibilities $k\neq j_1$ and $k=j_1$ exhaust all possibilities for $k\in[K]$, the stated identity follows.

It is worth noting that the conclusion is independent of the particular deterministic tie-breaking rule used to choose $j_1$: if the minimum value $m_1$ is attained at multiple indices, then selecting any one of them as $j_1$ still makes $m_2$ equal to the minimum over the remaining indices, and the formula above remains valid. The role of consistent tie-breaking is only to make the notation unambiguous.
\end{proof}

\begin{corollary}[DP time complexity]
\label{cor:dp_complexity}
All values
\[
\mathrm{dp}[t,s,k],
\qquad
t\in[T],\; s\in\{0,\dots,S\},\; k\in[K],
\]
can be computed in total time $O(TSK)$. If one stores backpointers sufficient to reconstruct an optimal switching path, then the memory requirement is $O(TSK)$. If one is interested only in the oracle value and not in path reconstruction, then the memory can be reduced to $O(SK)$ by keeping only the previous time slice in the dynamic program.
\end{corollary}

\begin{proof}
We begin with the cost of evaluating one dynamic-programming layer. Fix $(t,s)$ with $t\ge 2$ and $s\ge 1$. By Lemma~\ref{lem:dp_correctness}, the recurrence for each $k\in[K]$ is
\[
\mathrm{dp}[t,s,k]
=
\ell_{t,k}
+
\min\Big\{
\mathrm{dp}[t-1,s,k],\;
\min_{j\neq k}\mathrm{dp}[t-1,s-1,j]
\Big\}.
\]
A naive implementation would evaluate the constrained minimum
\[
\min_{j\neq k}\mathrm{dp}[t-1,s-1,j]
\]
separately for each $k$, which would cost $O(K)$ per state $k$ and therefore $O(K^2)$ for the whole layer $(t,s)$. The purpose of Lemma~\ref{lem:best_second_best} is precisely to avoid this quadratic dependence.

Indeed, for fixed $(t,s)$, define
\[
a_j \triangleq \mathrm{dp}[t-1,s-1,j], \qquad j\in[K].
\]
A single pass over the $K$ values $\{a_j\}_{j=1}^K$ suffices to compute:
\[
m_1=\min_{j\in[K]} a_j,
\qquad
j_1\in\arg\min_{j\in[K]} a_j,
\qquad
m_2=\min_{j\in[K]\setminus\{j_1\}} a_j.
\]
This preprocessing requires $O(K)$ time. Once $(m_1,j_1,m_2)$ have been computed, Lemma~\ref{lem:best_second_best} implies that for every $k\in[K]$,
\[
\min_{j\neq k}\mathrm{dp}[t-1,s-1,j]
=
\begin{cases}
m_1, & k\neq j_1,\\
m_2, & k=j_1.
\end{cases}
\]
Hence the constrained predecessor minimum for each $k$ is available in $O(1)$ time. Substituting this value into the recurrence shows that each entry $\mathrm{dp}[t,s,k]$ can then be computed in $O(1)$ time. Since there are $K$ possible values of $k$, the total cost of filling the whole layer $(t,s)$ is therefore $O(K)$.

The same conclusion is immediate for the boundary slice $s=0$. In that case, switching is disallowed, so the recurrence reduces to the stay-at-the-same-expert transition only, and the whole layer can again be computed in $O(K)$ time. Thus every pair $(t,s)$ contributes $O(K)$ work.

There are $T$ time indices, $S+1$ switch-budget indices, and for each such pair the computational cost is $O(K)$. Consequently, the total running time is
\[
O\big(T(S+1)K\big)=O(TSK).
\]

We now turn to memory complexity. If the objective is to recover not only the oracle value but also an optimal path $\pi^\star_{1:T}$, then one must store, for each state $(t,s,k)$, both the value $\mathrm{dp}[t,s,k]$ and sufficient predecessor information to identify which state at time $t-1$ attains the optimum. This requires storage proportional to the total number of states, namely
\[
O\big(T(S+1)K\big)=O(TSK).
\]

If, on the other hand, one is interested only in the oracle value
\[
L_T^{\mathrm{sw}}(S)=\min_{k\in[K]}\mathrm{dp}[T,S,k],
\]
then the full three-dimensional table need not be retained. The recurrence for time $t$ depends only on values from time $t-1$, namely the slices
\[
\{\mathrm{dp}[t-1,s,k]\}_{s,k}
\quad\text{and}\quad
\{\mathrm{dp}[t-1,s-1,k]\}_{s,k}.
\]
Therefore it is sufficient to store the current and previous time slices, or equivalently to overwrite one slice with the next as the computation proceeds. Each time slice contains $(S+1)K$ values, so the required storage is
\[
O\big((S+1)K\big)=O(SK).
\]

This proves the stated time and memory bounds.
\end{proof}

\subsection{C. A sharp exponential-weights inequality (Hoeffding lemma in full detail)}
\label{app:hoeffding}

The regret analysis ultimately hinges on a one-step inequality relating the learner's expected loss under $w_t$ to a log-partition function.
We give a complete proof starting from first principles.

\begin{lemma}[Variance bound for bounded random variables]
\label{lem:var_bound}
Let $Y$ be a real-valued random variable with support contained in $[0,1]$. Then
\[
\mathrm{Var}(Y)\le \frac14.
\]
More generally, if $X$ is a real-valued random variable with support contained in $[a,b]$, then
\[
\mathrm{Var}(X)\le \frac{(b-a)^2}{4}.
\]
\end{lemma}

\begin{proof}
We first consider the normalized case in which $Y$ takes values in $[0,1]$. Write
\[
\mu \triangleq \mathbb{E}[Y].
\]
Since $0\le Y\le 1$ almost surely, one has $Y^2\le Y$ almost surely. Taking expectations yields
\[
\mathbb{E}[Y^2]\le \mathbb{E}[Y]=\mu.
\]
Therefore
\[
\mathrm{Var}(Y)
=
\mathbb{E}[Y^2]-\big(\mathbb{E}[Y]\big)^2
\le
\mu-\mu^2
=
\mu(1-\mu).
\]
It remains to bound the quadratic expression $\mu(1-\mu)$ over $\mu\in[0,1]$. Since
\[
\mu(1-\mu)=\frac14-\left(\mu-\frac12\right)^2,
\]
we obtain
\[
\mu(1-\mu)\le \frac14,
\]
and hence
\[
\mathrm{Var}(Y)\le \frac14.
\]

We now turn to the general case in which $X\in[a,b]$ almost surely. If $a=b$, then $X$ is almost surely constant and the claim is trivial. Assume therefore that $b>a$, and define the affine rescaling
\[
Y \triangleq \frac{X-a}{b-a}.
\]
Then $Y\in[0,1]$ almost surely, so by the first part of the proof,
\[
\mathrm{Var}(Y)\le \frac14.
\]
Since
\[
X=a+(b-a)Y,
\]
the variance transformation rule for affine maps gives
\[
\mathrm{Var}(X)=(b-a)^2\mathrm{Var}(Y).
\]
Combining the two displays yields
\[
\mathrm{Var}(X)\le (b-a)^2\cdot \frac14 = \frac{(b-a)^2}{4}.
\]
This proves the claim.
\end{proof}

\begin{lemma}[Hoeffding's lemma]
\label{lem:hoeffding}
Let $X$ be a real-valued random variable such that $X\in[a,b]$ almost surely. Then, for every $\eta\in\mathbb{R}$,
\[
\log \mathbb{E}\!\left[\exp\!\big(\eta(X-\mathbb{E}[X])\big)\right]
\;\le\;
\frac{\eta^2(b-a)^2}{8}.
\]
\end{lemma}

\begin{proof}
Set
\[
\mu \triangleq \mathbb{E}[X],
\qquad
Z \triangleq X-\mu.
\]
Then $Z$ is almost surely supported on the interval $[a-\mu,b-\mu]$, and in particular $Z$ is bounded. Define the logarithmic moment generating function
\[
\psi(\eta)
\triangleq
\log \mathbb{E}\!\left[e^{\eta Z}\right]
=
\log \mathbb{E}\!\left[\exp\!\big(\eta(X-\mu)\big)\right],
\qquad \eta\in\mathbb{R}.
\]
Because $Z$ is bounded, the function $\eta\mapsto \mathbb{E}[e^{\eta Z}]$ is finite and infinitely differentiable on all of $\mathbb{R}$, and differentiation may be performed under the expectation sign. It follows immediately that
\[
\psi(0)=\log \mathbb{E}[1]=0,
\]
and
\[
\psi'(0)
=
\frac{\mathbb{E}[Z e^{0\cdot Z}]}{\mathbb{E}[e^{0\cdot Z}]}
=
\mathbb{E}[Z]
=
\mathbb{E}[X-\mu]
=
0.
\]

We now compute the first and second derivatives of $\psi$. Writing
\[
M(\eta)\triangleq \mathbb{E}[e^{\eta Z}],
\]
we have $\psi(\eta)=\log M(\eta)$, and therefore
\[
\psi'(\eta)=\frac{M'(\eta)}{M(\eta)}
=
\frac{\mathbb{E}[Z e^{\eta Z}]}{\mathbb{E}[e^{\eta Z}]}.
\]
It is convenient to introduce the exponentially tilted probability measure $\mathbb{P}_\eta$ defined by
\[
\frac{d\mathbb{P}_\eta}{d\mathbb{P}}
=
\frac{e^{\eta Z}}{\mathbb{E}[e^{\eta Z}]},
\]
which is well defined because $e^{\eta Z}>0$ almost surely and $\mathbb{E}[e^{\eta Z}]<\infty$. Denoting expectation under $\mathbb{P}_\eta$ by $\mathbb{E}_\eta$, the previous display may be rewritten as
\[
\psi'(\eta)=\mathbb{E}_\eta[Z].
\]
Differentiating once more gives
\[
\psi''(\eta)
=
\frac{\mathbb{E}[Z^2 e^{\eta Z}]\mathbb{E}[e^{\eta Z}] - \big(\mathbb{E}[Z e^{\eta Z}]\big)^2}{\big(\mathbb{E}[e^{\eta Z}]\big)^2}.
\]
Recognizing again the tilted expectation, this becomes
\[
\psi''(\eta)
=
\mathbb{E}_\eta[Z^2]-\big(\mathbb{E}_\eta[Z]\big)^2
=
\mathrm{Var}_\eta(Z).
\]
Since $Z=X-\mu$ differs from $X$ only by a deterministic shift, variance is translation invariant, and thus
\[
\mathrm{Var}_\eta(Z)=\mathrm{Var}_\eta(X).
\]
Accordingly,
\[
\psi''(\eta)=\mathrm{Var}_\eta(X).
\]

Next we bound this tilted variance uniformly in $\eta$. Since $X\in[a,b]$ almost surely under $\mathbb{P}$, the same support constraint remains valid under every tilted measure $\mathbb{P}_\eta$, because $\mathbb{P}_\eta$ is absolutely continuous with respect to $\mathbb{P}$. Hence $X\in[a,b]$ almost surely under $\mathbb{P}_\eta$ as well. Applying Lemma~\ref{lem:var_bound} under the measure $\mathbb{P}_\eta$ yields
\[
\psi''(\eta)=\mathrm{Var}_\eta(X)\le \frac{(b-a)^2}{4},
\qquad \forall \eta\in\mathbb{R}.
\]

We now integrate this second-derivative bound. Since $\psi$ is twice continuously differentiable, Taylor's formula with exact integral remainder gives, for every $\eta\in\mathbb{R}$,
\[
\psi(\eta)
=
\psi(0)+\eta\psi'(0)+\int_0^\eta (\eta-u)\psi''(u)\,du.
\]
Using $\psi(0)=0$ and $\psi'(0)=0$, we obtain
\[
\psi(\eta)
=
\int_0^\eta (\eta-u)\psi''(u)\,du.
\]
Invoking the uniform bound on $\psi''$ gives
\[
\psi(\eta)
\le
\frac{(b-a)^2}{4}\int_0^\eta (\eta-u)\,du.
\]
If $\eta\ge 0$, then
\[
\int_0^\eta (\eta-u)\,du=\frac{\eta^2}{2}.
\]
If $\eta<0$, the same identity still holds, since
\[
\int_0^\eta (\eta-u)\,du
=
-\int_\eta^0 (\eta-u)\,du
=
\frac{\eta^2}{2}.
\]
Therefore, for all $\eta\in\mathbb{R}$,
\[
\psi(\eta)\le \frac{(b-a)^2}{4}\cdot \frac{\eta^2}{2}
=
\frac{\eta^2(b-a)^2}{8}.
\]
Recalling the definition of $\psi$, this is exactly
\[
\log \mathbb{E}\!\left[\exp\!\big(\eta(X-\mathbb{E}[X])\big)\right]
\le
\frac{\eta^2(b-a)^2}{8}.
\]
This proves the claim.
\end{proof}

\begin{corollary}[One-step log-sum-exp inequality]
\label{cor:one_step}
Let $p\in\Delta_K$ and let $(x_1,\dots,x_K)\in[0,1]^K$. Then, for every $\eta\ge 0$,
\[
\log\Big(\sum_{k=1}^K p(k)e^{-\eta x_k}\Big)
\;\le\;
-\eta\sum_{k=1}^K p(k)x_k + \frac{\eta^2}{8}.
\]
Equivalently, for every $\eta>0$,
\[
\sum_{k=1}^K p(k)x_k
\;\le\;
-\frac{1}{\eta}\log\Big(\sum_{k=1}^K p(k)e^{-\eta x_k}\Big) + \frac{\eta}{8}.
\]
\end{corollary}

\begin{proof}
Let $X$ be a discrete random variable defined on the finite support $\{x_1,\dots,x_K\}$ by
\[
\mathbb{P}(X=x_k)=p(k), \qquad k\in[K].
\]
Since each $x_k$ lies in $[0,1]$, the random variable $X$ is almost surely supported on $[0,1]$. It follows that $-X$ is almost surely supported on $[-1,0]$, so Lemma~\ref{lem:hoeffding} applied to the random variable $-X$ yields, for every $\eta\in\mathbb{R}$,
\[
\log \mathbb{E}\!\left[\exp\!\big(\eta((-X)-\mathbb{E}[-X])\big)\right]
\le
\frac{\eta^2}{8},
\]
because the support interval has length
\[
0-(-1)=1.
\]
Using the identity $\mathbb{E}[-X]=-\mathbb{E}[X]$, this becomes
\[
\log \mathbb{E}\!\left[\exp\!\big(-\eta(X-\mathbb{E}[X])\big)\right]
\le
\frac{\eta^2}{8}.
\]
Expanding the centered term inside the exponential gives
\[
-\eta(X-\mathbb{E}[X])=-\eta X+\eta \mathbb{E}[X].
\]
Therefore,
\[
\mathbb{E}\!\left[\exp\!\big(-\eta(X-\mathbb{E}[X])\big)\right]
=
\mathbb{E}\!\left[e^{-\eta X+\eta \mathbb{E}[X]}\right]
=
e^{\eta\mathbb{E}[X]}\,\mathbb{E}[e^{-\eta X}],
\]
and taking logarithms yields
\[
\log \mathbb{E}\!\left[\exp\!\big(-\eta(X-\mathbb{E}[X])\big)\right]
=
\eta\mathbb{E}[X]+\log \mathbb{E}[e^{-\eta X}].
\]
Combining this identity with the previous bound gives
\[
\eta\mathbb{E}[X]+\log \mathbb{E}[e^{-\eta X}]
\le
\frac{\eta^2}{8},
\]
or equivalently,
\[
\log \mathbb{E}[e^{-\eta X}]
\le
-\eta\mathbb{E}[X]+\frac{\eta^2}{8}.
\]

We now rewrite both terms in discrete form. By construction of $X$,
\[
\mathbb{E}[X]=\sum_{k=1}^K p(k)x_k
\]
and
\[
\mathbb{E}[e^{-\eta X}]
=
\sum_{k=1}^K p(k)e^{-\eta x_k}.
\]
Substituting these identities into the previous inequality gives
\[
\log\Big(\sum_{k=1}^K p(k)e^{-\eta x_k}\Big)
\le
-\eta\sum_{k=1}^K p(k)x_k+\frac{\eta^2}{8},
\]
which is the first statement.

If $\eta>0$, this inequality can be rearranged by dividing both sides by $-\eta$; since division by a negative quantity reverses the inequality sign, we obtain
\[
\sum_{k=1}^K p(k)x_k
\le
-\frac{1}{\eta}\log\Big(\sum_{k=1}^K p(k)e^{-\eta x_k}\Big)+\frac{\eta}{8}.
\]
This is exactly the equivalent form stated in the corollary.
\end{proof}

\subsection{D. Generalized Share as exponential weights over paths (fully rigorous equivalence)}
\label{app:path_equivalence}

This subsection establishes the structural identity used in the main theorem: Generalized Share is the marginal of exponential weights over the class of expert paths under a Markov prior.

\subsubsection{D.1 Transition kernel and path prior}
Given restart parameters $(\rho_t,q_t)$, define the time-varying transition kernel
\begin{equation}
A_t(i\to j)
\;\triangleq\;
(1-\rho_t)\mathbf{1}\{i=j\} + \rho_t q_t(j),
\qquad t=1,\dots,T-1.
\label{eq:At_def}
\end{equation}
Each row of $A_t$ sums to $1$, hence $A_t$ is row-stochastic.

Define the induced (possibly data-dependent) Markov prior over paths $\pi_{1:T}\in[K]^T$:
\begin{equation}
\mathbb{P}(\pi_{1:T})
\;\triangleq\;
w_1(\pi_1)\prod_{t=1}^{T-1}A_t(\pi_t\to\pi_{t+1}).
\label{eq:path_prior_app}
\end{equation}

\begin{lemma}[The Markov prior is normalized]
\label{lem:prior_normalized}
Suppose $w_1\in\Delta_K$ and, for each $t\in\{1,\dots,T-1\}$, the matrix $A_t$ is row-stochastic, i.e.,
\[
A_t(i\to j)\ge 0
\quad\text{for all } i,j\in[K],
\qquad
\sum_{j=1}^K A_t(i\to j)=1
\quad\text{for all } i\in[K].
\]
Define
\[
\mathbb{P}(\pi_{1:T})
\;\triangleq\;
w_1(\pi_1)\prod_{t=1}^{T-1}A_t(\pi_t\to\pi_{t+1}),
\qquad \pi_{1:T}\in[K]^T.
\]
Then $\mathbb{P}$ is a probability distribution on $[K]^T$, that is,
\[
\sum_{\pi_{1:T}\in[K]^T}\mathbb{P}(\pi_{1:T})=1.
\]
\end{lemma}

\begin{proof}
By construction, $\mathbb{P}(\pi_{1:T})\ge 0$ for every path $\pi_{1:T}\in[K]^T$, since $w_1(\pi_1)\ge 0$ and every transition probability $A_t(\pi_t\to\pi_{t+1})$ is nonnegative. It therefore remains only to verify that the total mass is equal to one.

Starting from the definition of $\mathbb{P}$, we have
\[
\sum_{\pi_{1:T}\in[K]^T}\mathbb{P}(\pi_{1:T})
=
\sum_{\pi_1=1}^K\cdots\sum_{\pi_T=1}^K
w_1(\pi_1)\prod_{t=1}^{T-1}A_t(\pi_t\to\pi_{t+1}).
\]
Since all sums are finite, we may rearrange them freely. Pulling out the factor depending only on $\pi_1$ gives
\[
\sum_{\pi_{1:T}\in[K]^T}\mathbb{P}(\pi_{1:T})
=
\sum_{\pi_1=1}^K
w_1(\pi_1)
\sum_{\pi_2=1}^K\cdots\sum_{\pi_T=1}^K
\prod_{t=1}^{T-1}A_t(\pi_t\to\pi_{t+1}).
\]
We now evaluate the inner sums successively from right to left. Fix $\pi_1,\dots,\pi_{T-1}$. Then, by row-stochasticity of $A_{T-1}$,
\[
\sum_{\pi_T=1}^K A_{T-1}(\pi_{T-1}\to\pi_T)=1.
\]
Hence
\[
\sum_{\pi_T=1}^K \prod_{t=1}^{T-1}A_t(\pi_t\to\pi_{t+1})
=
\Big(\prod_{t=1}^{T-2}A_t(\pi_t\to\pi_{t+1})\Big)
\sum_{\pi_T=1}^K A_{T-1}(\pi_{T-1}\to\pi_T)
=
\prod_{t=1}^{T-2}A_t(\pi_t\to\pi_{t+1}).
\]
Substituting this back, the sum over $\pi_T$ disappears. Repeating exactly the same argument for $\pi_{T-1},\pi_{T-2},\dots,\pi_2$, we successively obtain
\[
\sum_{\pi_2=1}^K\cdots\sum_{\pi_T=1}^K
\prod_{t=1}^{T-1}A_t(\pi_t\to\pi_{t+1})
=1
\]
for every fixed initial state $\pi_1$. Therefore,
\[
\sum_{\pi_{1:T}\in[K]^T}\mathbb{P}(\pi_{1:T})
=
\sum_{\pi_1=1}^K w_1(\pi_1).
\]
Finally, since $w_1\in\Delta_K$, its coordinates sum to one, and hence
\[
\sum_{\pi_1=1}^K w_1(\pi_1)=1.
\]
Combining the displays above yields
\[
\sum_{\pi_{1:T}\in[K]^T}\mathbb{P}(\pi_{1:T})=1.
\]
Thus $\mathbb{P}$ is normalized and therefore defines a probability distribution on the path space $[K]^T$.
\end{proof}

\subsubsection{D.2 Forward masses, normalized weights, and the exact share update}

To keep the indexing consistent with a horizon of $T$ losses and $T-1$ transitions, we define the forward masses only up to time $T$, and then define a terminal partition function after the final loss has been incorporated.

Let $m_1=w_1$, and for each $t=1,\dots,T-1$ define
\[
v_t(i)\triangleq m_t(i)e^{-\eta_t \ell_{t,i}},
\qquad
m_{t+1}(j)\triangleq \sum_{i=1}^K v_t(i)\,A_t(i\to j).
\]
For $t=1,\dots,T$, let
\[
\Phi_t\triangleq \sum_{k=1}^K m_t(k),
\qquad
w_t(k)\triangleq \frac{m_t(k)}{\Phi_t}\in\Delta_K.
\]
Finally, define the terminal partition function after round $T$ by
\begin{equation}
\Phi_{T+1}\triangleq \sum_{i=1}^K m_T(i)e^{-\eta_T \ell_{T,i}}.
\label{eq:terminal_partition}
\end{equation}

\begin{lemma}[Forward recursion equals Generalized Share]
\label{lem:forward_equals_share}
Assume that the transition kernel $A_t$ has the share form
\begin{equation}
A_t(i\to j)
=
(1-\rho_t)\mathbf{1}\{i=j\}+\rho_t q_t(j),
\qquad i,j\in[K],\quad t=1,\dots,T-1.
\label{eq:At_def_recalled}
\end{equation}
Let the forward masses $\{m_t\}_{t=1}^T$ be defined by
\[
m_1=w_1,
\qquad
v_t(i)\triangleq m_t(i)e^{-\eta_t \ell_{t,i}},
\qquad
m_{t+1}(j)\triangleq \sum_{i=1}^K v_t(i)A_t(i\to j),
\quad t=1,\dots,T-1,
\]
and let
\[
\Phi_t\triangleq \sum_{k=1}^K m_t(k),
\qquad
w_t(k)\triangleq \frac{m_t(k)}{\Phi_t},
\qquad t=1,\dots,T.
\]
Then for every $t=1,\dots,T-1$, the normalized weights satisfy exactly the Generalized Share update:
\[
\bar w_{t+1}(k)
=
\frac{w_t(k)e^{-\eta_t \ell_{t,k}}}{\sum_{j=1}^K w_t(j)e^{-\eta_t \ell_{t,j}}},
\qquad
w_{t+1}
=
(1-\rho_t)\bar w_{t+1}+\rho_t q_t.
\]
Moreover,
\begin{equation}
\frac{\Phi_{t+1}}{\Phi_t}
=
\sum_{k=1}^K w_t(k)e^{-\eta_t \ell_{t,k}},
\qquad t=1,\dots,T-1.
\label{eq:phi_ratio_preterminal}
\end{equation}
\end{lemma}

\begin{proof}
Fix $t\in\{1,\dots,T-1\}$. Since $w_t(k)=m_t(k)/\Phi_t$, we have
\[
\frac{w_t(k)e^{-\eta_t \ell_{t,k}}}{\sum_{j=1}^K w_t(j)e^{-\eta_t \ell_{t,j}}}
=
\frac{\frac{m_t(k)}{\Phi_t}e^{-\eta_t \ell_{t,k}}}
{\sum_{j=1}^K \frac{m_t(j)}{\Phi_t}e^{-\eta_t \ell_{t,j}}}
=
\frac{m_t(k)e^{-\eta_t \ell_{t,k}}}{\sum_{j=1}^K m_t(j)e^{-\eta_t \ell_{t,j}}}
=
\frac{v_t(k)}{\sum_{j=1}^K v_t(j)}.
\]
Hence
\[
\bar w_{t+1}(k)=\frac{v_t(k)}{\sum_{j=1}^K v_t(j)}.
\]

Next, by definition of the forward recursion and \eqref{eq:At_def_recalled},
\[
m_{t+1}(j)
=
\sum_{i=1}^K v_t(i)\Big((1-\rho_t)\mathbf{1}\{i=j\}+\rho_t q_t(j)\Big)
=
(1-\rho_t)v_t(j)+\rho_t q_t(j)\sum_{i=1}^K v_t(i).
\]
Summing over $j$ gives
\[
\Phi_{t+1}
=
\sum_{j=1}^K m_{t+1}(j)
=
(1-\rho_t)\sum_{j=1}^K v_t(j)+\rho_t\Big(\sum_{j=1}^K q_t(j)\Big)\Big(\sum_{i=1}^K v_t(i)\Big)
=
\sum_{i=1}^K v_t(i),
\]
because $q_t\in\Delta_K$.

Therefore,
\[
w_{t+1}(j)
=
\frac{m_{t+1}(j)}{\Phi_{t+1}}
=
\frac{(1-\rho_t)v_t(j)+\rho_t q_t(j)\sum_{i=1}^K v_t(i)}
{\sum_{i=1}^K v_t(i)}
=
(1-\rho_t)\frac{v_t(j)}{\sum_{i=1}^K v_t(i)}+\rho_t q_t(j),
\]
and hence
\[
w_{t+1}(j)=(1-\rho_t)\bar w_{t+1}(j)+\rho_t q_t(j).
\]
This proves the share update.

Finally,
\[
\frac{\Phi_{t+1}}{\Phi_t}
=
\frac{\sum_{i=1}^K v_t(i)}{\Phi_t}
=
\sum_{i=1}^K \frac{m_t(i)}{\Phi_t}e^{-\eta_t \ell_{t,i}}
=
\sum_{i=1}^K w_t(i)e^{-\eta_t \ell_{t,i}},
\]
which is \eqref{eq:phi_ratio_preterminal}.
\end{proof}
\subsubsection{D.3 Path partition function identity}

The next lemma gives the path-sum representation of the forward masses and the terminal partition function for a horizon of $T$ losses and $T-1$ transitions.

\begin{lemma}[Path-sum representation of the forward masses and terminal partition]
\label{lem:path_sum}
For each $t\in\{1,\dots,T\}$ and each $j\in[K]$,
\begin{equation}
m_t(j)
=
\sum_{\substack{\pi_{1:t}\in[K]^t:\\ \pi_t=j}}
w_1(\pi_1)
\Big(\prod_{\tau=1}^{t-1}A_\tau(\pi_\tau\to \pi_{\tau+1})\Big)
\exp\!\Big(-\sum_{\tau=1}^{t-1}\eta_\tau \ell_{\tau,\pi_\tau}\Big),
\label{eq:mt_path_sum}
\end{equation}
where the empty product and empty sum at $t=1$ are interpreted as $1$ and $0$, respectively. Consequently,
\begin{equation}
\Phi_{T+1}
=
\sum_{\pi_{1:T}\in[K]^T}
w_1(\pi_1)
\Big(\prod_{t=1}^{T-1}A_t(\pi_t\to\pi_{t+1})\Big)
\exp\!\Big(-\sum_{t=1}^{T}\eta_t\ell_{t,\pi_t}\Big).
\label{eq:Phi_path_sum}
\end{equation}
\end{lemma}

\begin{proof}
We prove \eqref{eq:mt_path_sum} by induction on $t$.

For $t=1$, the claim is immediate because $m_1=w_1$ and there are no previous losses or transitions:
\[
m_1(j)=w_1(j)
=
\sum_{\substack{\pi_1\in[K]:\\ \pi_1=j}}
w_1(\pi_1).
\]

Now suppose the claim holds for some $t\in\{1,\dots,T-1\}$. Then for any $j\in[K]$,
\[
m_{t+1}(j)
=
\sum_{i=1}^K m_t(i)e^{-\eta_t \ell_{t,i}}A_t(i\to j).
\]
Substituting the induction hypothesis for $m_t(i)$ gives
\[
m_{t+1}(j)
=
\sum_{i=1}^K
\sum_{\substack{\pi_{1:t}\in[K]^t:\\ \pi_t=i}}
w_1(\pi_1)
\Big(\prod_{\tau=1}^{t-1}A_\tau(\pi_\tau\to\pi_{\tau+1})\Big)
\exp\!\Big(-\sum_{\tau=1}^{t-1}\eta_\tau \ell_{\tau,\pi_\tau}\Big)
e^{-\eta_t \ell_{t,i}}A_t(i\to j).
\]
Since $\pi_t=i$ in the inner sum, we can rewrite this as
\[
m_{t+1}(j)
=
\sum_{i=1}^K
\sum_{\substack{\pi_{1:t}\in[K]^t:\\ \pi_t=i}}
w_1(\pi_1)
\Big(\prod_{\tau=1}^{t}A_\tau(\pi_\tau\to\pi_{\tau+1})\Big)
\exp\!\Big(-\sum_{\tau=1}^{t}\eta_\tau \ell_{\tau,\pi_\tau}\Big),
\]
with the convention $\pi_{t+1}=j$. Summing over all $i$ is equivalent to summing over all paths $\pi_{1:t+1}\in[K]^{t+1}$ with terminal state $\pi_{t+1}=j$, hence
\[
m_{t+1}(j)
=
\sum_{\substack{\pi_{1:t+1}\in[K]^{t+1}:\\ \pi_{t+1}=j}}
w_1(\pi_1)
\Big(\prod_{\tau=1}^{t}A_\tau(\pi_\tau\to\pi_{\tau+1})\Big)
\exp\!\Big(-\sum_{\tau=1}^{t}\eta_\tau \ell_{\tau,\pi_\tau}\Big).
\]
This proves \eqref{eq:mt_path_sum}.

Finally, by the definition of the terminal partition function \eqref{eq:terminal_partition},
\[
\Phi_{T+1}
=
\sum_{j=1}^K m_T(j)e^{-\eta_T \ell_{T,j}}.
\]
Applying \eqref{eq:mt_path_sum} with $t=T$ and then absorbing the final factor $e^{-\eta_T \ell_{T,j}}$ into the exponent yields
\[
\Phi_{T+1}
=
\sum_{j=1}^K
\sum_{\substack{\pi_{1:T}\in[K]^T:\\ \pi_T=j}}
w_1(\pi_1)
\Big(\prod_{t=1}^{T-1}A_t(\pi_t\to\pi_{t+1})\Big)
\exp\!\Big(-\sum_{t=1}^{T}\eta_t\ell_{t,\pi_t}\Big).
\]
Summing over $j$ removes the terminal-state constraint and gives \eqref{eq:Phi_path_sum}.
\end{proof}

\subsection{E. Proof of the main regret bound (Theorem~\ref{thm:pcgs_pathwise})}
\label{app:proof_main}

\begin{proof}
We prove the weighted pathwise bound first; the constant-learning-rate specialization then follows immediately.

Let the forward masses $\{m_t\}_{t=1}^T$ and the terminal partition function $\Phi_{T+1}$ be defined as in Appendix~\ref{app:path_equivalence}. In particular,
\[
\Phi_t=\sum_{k=1}^K m_t(k),\qquad
w_t(i)=\frac{m_t(i)}{\Phi_t},\qquad t=1,\dots,T,
\]
and, for $t=1,\dots,T-1$,
\[
\frac{\Phi_{t+1}}{\Phi_t}
=
\sum_{i=1}^K w_t(i)e^{-\eta_t \tilde\ell_{t,i}}
\]
by \eqref{eq:phi_ratio_preterminal}. For the last round,
\[
\frac{\Phi_{T+1}}{\Phi_T}
=
\sum_{i=1}^K w_T(i)e^{-\eta_T \tilde\ell_{T,i}}
\]
holds directly from the definition of the terminal partition function \eqref{eq:terminal_partition}. Therefore, for every $t=1,\dots,T$,
\begin{equation}
\frac{\Phi_{t+1}}{\Phi_t}
=
\sum_{i=1}^K w_t(i)e^{-\eta_t \tilde\ell_{t,i}}.
\label{eq:phi_ratio_all_t}
\end{equation}

We now apply Corollary~\ref{cor:one_step} with
\[
p(i)=w_t(i),\qquad x_i=\tilde\ell_{t,i},\qquad \eta=\eta_t.
\]
Since $\tilde\ell_{t,i}\in[0,1]$ for all $i$, we obtain
\[
\log\!\Big(\sum_{i=1}^K w_t(i)e^{-\eta_t \tilde\ell_{t,i}}\Big)
\le
-\eta_t \sum_{i=1}^K w_t(i)\tilde\ell_{t,i}
+
\frac{\eta_t^2}{8}.
\]
Using \eqref{eq:phi_ratio_all_t}, this becomes
\[
\log\!\Big(\frac{\Phi_{t+1}}{\Phi_t}\Big)
\le
-\eta_t\langle w_t,\tilde\ell_t\rangle
+
\frac{\eta_t^2}{8},
\]
or equivalently,
\[
\eta_t\langle w_t,\tilde\ell_t\rangle
\le
-\log\!\Big(\frac{\Phi_{t+1}}{\Phi_t}\Big)
+
\frac{\eta_t^2}{8}.
\]
Summing over $t=1,\dots,T$ yields
\[
\sum_{t=1}^T \eta_t\langle w_t,\tilde\ell_t\rangle
\le
-\sum_{t=1}^T \log\!\Big(\frac{\Phi_{t+1}}{\Phi_t}\Big)
+
\frac{1}{8}\sum_{t=1}^T \eta_t^2.
\]
The logarithmic terms telescope:
\[
-\sum_{t=1}^T \log\!\Big(\frac{\Phi_{t+1}}{\Phi_t}\Big)
=
-\log \Phi_{T+1}+\log \Phi_1.
\]
Since $m_1=w_1$ and $w_1\in\Delta_K$,
\[
\Phi_1=\sum_{k=1}^K w_1(k)=1,
\]
so $\log\Phi_1=0$. Hence
\begin{equation}
\sum_{t=1}^T \eta_t\langle w_t,\tilde\ell_t\rangle
\le
-\log \Phi_{T+1}
+
\frac{1}{8}\sum_{t=1}^T \eta_t^2.
\label{eq:weighted_upper_partition}
\end{equation}

Next, by the path-sum representation \eqref{eq:Phi_path_sum} from Lemma~\ref{lem:path_sum}, for any comparator path $\pi_{1:T}\in[K]^T$,
\[
\Phi_{T+1}
\ge
w_1(\pi_1)
\Big(\prod_{t=1}^{T-1}A_t(\pi_t\to\pi_{t+1})\Big)
\exp\!\Big(-\sum_{t=1}^T \eta_t \tilde\ell_{t,\pi_t}\Big).
\]
Taking negative logarithms gives
\[
-\log \Phi_{T+1}
\le
\sum_{t=1}^T \eta_t \tilde\ell_{t,\pi_t}
-
\log w_1(\pi_1)
-
\sum_{t=1}^{T-1}\log A_t(\pi_t\to\pi_{t+1}).
\]
Substituting this into \eqref{eq:weighted_upper_partition} yields
\[
\sum_{t=1}^T \eta_t\langle w_t,\tilde\ell_t\rangle
\le
\sum_{t=1}^T \eta_t \tilde\ell_{t,\pi_t}
-
\log w_1(\pi_1)
-
\sum_{t=1}^{T-1}\log A_t(\pi_t\to\pi_{t+1})
+
\frac{1}{8}\sum_{t=1}^T \eta_t^2.
\]
Rearranging proves
\[
\sum_{t=1}^T \eta_t\Big(\langle w_t,\tilde\ell_t\rangle-\tilde\ell_{t,\pi_t}\Big)
\le
-\log w_1(\pi_1)
-
\sum_{t=1}^{T-1}\log A_t(\pi_t\to\pi_{t+1})
+
\frac{1}{8}\sum_{t=1}^T \eta_t^2,
\]
which is \eqref{eq:thm_tv_eta}.

Finally, if $\eta_t\equiv\eta>0$, then \eqref{eq:thm_tv_eta} becomes
\[
\eta\sum_{t=1}^T \Big(\langle w_t,\tilde\ell_t\rangle-\tilde\ell_{t,\pi_t}\Big)
\le
-\log w_1(\pi_1)
-
\sum_{t=1}^{T-1}\log A_t(\pi_t\to\pi_{t+1})
+
\frac{\eta^2 T}{8}.
\]
Dividing by $\eta$ yields \eqref{eq:thm_constant_eta}. The argument is entirely pathwise and requires no probabilistic assumptions on the loss sequence beyond the measurability conditions imposed on the update controls.
\end{proof}

\subsection{F. Fixed Share as a special case (Corollary~\ref{cor:fixedshare})}
\label{app:fixedshare}

\begin{proof}[Proof of Corollary~\ref{cor:fixedshare}]
Under the assumptions of the corollary, the restart distribution is uniform,
\[
q_t\equiv u,
\qquad
u(k)=\frac{1}{K},
\]
the restart intensity is constant, $\rho_t\equiv \rho\in(0,1)$, and the initial mixture is uniform,
\[
w_1\equiv \frac{\mathbf{1}}{K}.
\]
Therefore the transition kernel does not depend on $t$ and takes the form
\[
A(i\to j)
=
(1-\rho)\mathbf{1}\{i=j\}+\frac{\rho}{K},
\qquad i,j\in[K].
\]

Fix any comparator path $\pi\in\Pi_S$. By definition, $\pi$ incurs at most $S$ switches over the $T-1$ transitions, that is,
\[
\#\mathrm{sw}(\pi)
=
\sum_{t=1}^{T-1}\mathbf{1}\{\pi_{t+1}\neq \pi_t\}
\le S.
\]
We will upper bound the transition complexity term
\[
-\sum_{t=1}^{T-1}\log A(\pi_t\to\pi_{t+1})
\]
appearing in Theorem~\ref{thm:pcgs_pathwise}.

For every transition time $t\in\{1,\dots,T-1\}$, the value of $A(\pi_t\to\pi_{t+1})$ depends only on whether the path stays on the same expert or switches to a different one. If $\pi_{t+1}\neq \pi_t$, then the diagonal term vanishes and one has
\[
A(\pi_t\to\pi_{t+1})=\frac{\rho}{K}.
\]
Hence each such transition contributes
\[
-\log A(\pi_t\to\pi_{t+1})
=
-\log(\rho/K)
=
\log(K/\rho).
\]
Since the path contains at most $S$ switch transitions, the total contribution from all switch times is bounded by
\[
\sum_{t:\,\pi_{t+1}\neq \pi_t}
-\log A(\pi_t\to\pi_{t+1})
\le
S\log(K/\rho).
\]

If instead $\pi_{t+1}=\pi_t$, then
\[
A(\pi_t\to\pi_{t+1})
=
(1-\rho)+\frac{\rho}{K}.
\]
Because $\rho/K\ge 0$, this implies
\[
A(\pi_t\to\pi_{t+1})\ge 1-\rho,
\]
and therefore
\[
-\log A(\pi_t\to\pi_{t+1})
\le
-\log(1-\rho)
=
\log\frac{1}{1-\rho}.
\]
There are at most $T-1$ such terms, so the total contribution from all stay transitions is bounded by
\[
\sum_{t:\,\pi_{t+1}= \pi_t}
-\log A(\pi_t\to\pi_{t+1})
\le
(T-1)\log\frac{1}{1-\rho}.
\]
This bound is slightly loose, since the exact number of stay transitions is $T-1-\#\mathrm{sw}(\pi)$, but the coarser expression above is sufficient for the corollary.

Combining the switch and stay contributions yields
\[
-\sum_{t=1}^{T-1}\log A(\pi_t\to\pi_{t+1})
\le
S\log\frac{K}{\rho}
+
(T-1)\log\frac{1}{1-\rho}.
\]
It remains to bound the initial-code term. Since $w_1$ is uniform on $[K]$,
\[
w_1(\pi_1)=\frac{1}{K},
\qquad
-\log w_1(\pi_1)=\log K.
\]

We now substitute these estimates into the constant-learning-rate form of Theorem~\ref{thm:pcgs_pathwise}. For every path $\pi\in\Pi_S$,
\[
\sum_{t=1}^T \langle w_t,\tilde\ell_t\rangle
-
\sum_{t=1}^T \tilde\ell_{t,\pi_t}
\le
\frac{1}{\eta}
\Big[
-\log w_1(\pi_1)
-
\sum_{t=1}^{T-1}\log A(\pi_t\to\pi_{t+1})
\Big]
+
\frac{\eta T}{8}.
\]
Using the bounds derived above, we obtain
\[
\sum_{t=1}^T \langle w_t,\tilde\ell_t\rangle
-
\sum_{t=1}^T \tilde\ell_{t,\pi_t}
\le
\frac{1}{\eta}
\Big[
\log K
+
S\log\frac{K}{\rho}
+
(T-1)\log\frac{1}{1-\rho}
\Big]
+
\frac{\eta T}{8},
\]
which is exactly \eqref{eq:cor_fixedshare}.

Finally, choosing $\rho\asymp S/T$ balances the switch term and the stay term in the usual way and recovers the standard $S\log K$-type dependence, up to familiar lower-order logarithmic factors. This is the classical Fixed Share scaling obtained as a special case of the more general \textsc{PCGS} bound.
\end{proof}

\subsection{G. Oracle-supervised training and switching complexity (Lemma~\ref{lem:ce_switch})}
\label{app:ce_switch}

Recall that for a fixed deterministic comparator path $\pi^\star_{1:T}\in[K]^T$, the effective switching complexity induced by the restart distribution sequence $\{q_t\}_{t=1}^{T-1}$ is defined by
\[
C_{\mathrm{eff}}(\pi^\star)
\triangleq
\sum_{t:\,\pi^\star_{t+1}\neq \pi^\star_t} -\log q_t(\pi^\star_{t+1}).
\]
We now justify the use of cross-entropy supervision for the $q$-head.

\begin{proof}[Proof of Lemma~\ref{lem:ce_switch}]
Let
\[
S^\star
\triangleq
\big\{t\in\{1,\dots,T-1\}:\pi^\star_{t+1}\neq \pi^\star_t\big\}
\]
denote the set of switch times of the comparator path $\pi^\star$. By definition of $C_{\mathrm{eff}}(\pi^\star)$,
\[
C_{\mathrm{eff}}(\pi^\star)
=
\sum_{t\in S^\star} -\log q_t(\pi^\star_{t+1}).
\]
Because each $q_t$ is a probability distribution on $[K]$, one has
\[
0< q_t(\pi^\star_{t+1}) \le 1
\]
whenever the logarithm is well defined, and therefore every summand satisfies
\[
-\log q_t(\pi^\star_{t+1})\ge 0.
\]
Since $S^\star\subseteq \{1,\dots,T-1\}$, enlarging the index set from $S^\star$ to all times can only increase the sum of these nonnegative terms. It follows that
\[
C_{\mathrm{eff}}(\pi^\star)
=
\sum_{t\in S^\star} -\log q_t(\pi^\star_{t+1})
\le
\sum_{t=1}^{T-1} -\log q_t(\pi^\star_{t+1}).
\]
The right-hand side is exactly the time-aggregated cross-entropy objective used to supervise the restart-distribution head, since by definition
\[
\mathcal{L}_q(t)
=
-\log q_t(\pi^\star_{t+1}).
\]
Hence
\[
\sum_{t=1}^{T-1}\mathcal{L}_q(t)
=
\sum_{t=1}^{T-1} -\log q_t(\pi^\star_{t+1})
\]
is a deterministic upper bound on the effective switching complexity $C_{\mathrm{eff}}(\pi^\star)$. In particular, minimizing the full cross-entropy objective necessarily controls the switch-dependent term appearing in the regret analysis, even though the supervision is imposed at all times rather than only at the switch times.

It remains to verify the uniform finiteness claim under $\varepsilon$-uniform mixing. Suppose that for every $t$ and every $k\in[K]$,
\[
q_t(k)\ge \frac{\varepsilon}{K}
\]
for some $\varepsilon>0$. Then for every $t$,
\[
q_t(\pi^\star_{t+1})\ge \frac{\varepsilon}{K},
\]
and therefore
\[
-\log q_t(\pi^\star_{t+1})
\le
\log\frac{K}{\varepsilon}.
\]
Summing this bound over the switch times gives
\[
C_{\mathrm{eff}}(\pi^\star)
=
\sum_{t\in S^\star} -\log q_t(\pi^\star_{t+1})
\le
|S^\star|\log\frac{K}{\varepsilon}.
\]
By definition, the cardinality of $S^\star$ is exactly the number of switches of the path:
\[
|S^\star|=\#\mathrm{sw}(\pi^\star).
\]
Consequently,
\[
C_{\mathrm{eff}}(\pi^\star)
\le
\#\mathrm{sw}(\pi^\star)\log\frac{K}{\varepsilon}.
\]

If, in particular, the comparator path satisfies $\#\mathrm{sw}(\pi^\star)\le S$, then one obtains the cleaner bound
\[
C_{\mathrm{eff}}(\pi^\star)\le S\log\frac{K}{\varepsilon}.
\]
Thus the effective switching complexity is always finite under $\varepsilon$-uniform mixing, and the cross-entropy training objective provides a deterministic upper bound on the switch-dependent code-length term of interest. This proves the lemma.
\end{proof}

\subsubsection{G.1 A stronger (information-theoretic) statement}
The lemma above is deterministic and already suffices for the theory-to-training alignment. For completeness, we also record the stronger conditional-entropy decomposition that explains what cross-entropy is optimizing \emph{in expectation} across tasks/sequences.

\begin{proposition}[Conditional entropy + KL decomposition for next-expert prediction]
\label{prop:ce_kl}
Let $Y_{t+1}$ be a $[K]$-valued random variable and let $\mathcal{G}_t$ be the $\sigma$-algebra representing the information available to the policy at time $t$. Define the conditional law
\[
p_t(j)\triangleq \mathbb{P}(Y_{t+1}=j\mid \mathcal{G}_t),
\qquad j\in[K].
\]
Let $q_t\in\Delta_K$ be any $\mathcal{G}_t$-measurable predicted distribution. Then, almost surely,
\[
\mathbb{E}\big[-\log q_t(Y_{t+1})\mid \mathcal{G}_t\big]
=
H(p_t)+\mathrm{KL}(p_t\|q_t),
\]
where
\[
H(p_t)\triangleq -\sum_{j=1}^K p_t(j)\log p_t(j)
\]
is the conditional entropy of $Y_{t+1}$ given $\mathcal{G}_t$, and
\[
\mathrm{KL}(p_t\|q_t)
\triangleq
\sum_{j=1}^K p_t(j)\log\frac{p_t(j)}{q_t(j)}
\]
is the conditional Kullback--Leibler divergence. Here we adopt the standard conventions
\[
0\log 0 = 0,
\qquad
0\log\frac{0}{q}=0 \ \text{ for } q>0,
\]
and $\mathrm{KL}(p_t\|q_t)=+\infty$ whenever $p_t(j)>0$ and $q_t(j)=0$ for some $j$.

In particular, the conditional expected log-loss is minimized almost surely by choosing
\[
q_t=p_t.
\]
Moreover, if $q_t(j)>0$ for all $j\in[K]$ almost surely, then this minimizer is unique almost surely.
\end{proposition}

\begin{proof}
Since $Y_{t+1}$ takes values in the finite set $[K]$ and $q_t$ is $\mathcal{G}_t$-measurable, the random variable $-\log q_t(Y_{t+1})$ is conditionally integrable whenever the displayed quantities are finite, and otherwise the identity is interpreted in the extended real sense. Conditioning on $\mathcal{G}_t$, we may evaluate the conditional expectation by summing against the conditional distribution of $Y_{t+1}$:
\[
\mathbb{E}\big[-\log q_t(Y_{t+1})\mid \mathcal{G}_t\big]
=
\sum_{j=1}^K \mathbb{P}(Y_{t+1}=j\mid \mathcal{G}_t)\,(-\log q_t(j)).
\]
By definition of $p_t(j)$, this becomes
\[
\mathbb{E}\big[-\log q_t(Y_{t+1})\mid \mathcal{G}_t\big]
=
\sum_{j=1}^K p_t(j)\,(-\log q_t(j)).
\]
We now add and subtract the term $\sum_{j=1}^K p_t(j)\log p_t(j)$. This yields
\[
\sum_{j=1}^K p_t(j)(-\log q_t(j))
=
\sum_{j=1}^K p_t(j)(-\log p_t(j))
+
\sum_{j=1}^K p_t(j)\log\frac{p_t(j)}{q_t(j)}.
\]
The first sum is exactly the conditional entropy $H(p_t)$, and the second is exactly the conditional divergence $\mathrm{KL}(p_t\|q_t)$. Therefore,
\[
\mathbb{E}\big[-\log q_t(Y_{t+1})\mid \mathcal{G}_t\big]
=
H(p_t)+\mathrm{KL}(p_t\|q_t)
\qquad \text{a.s.}
\]

To identify the minimizer, recall the elementary nonnegativity of Kullback--Leibler divergence:
\[
\mathrm{KL}(p_t\|q_t)\ge 0
\qquad \text{a.s.}
\]
Hence
\[
\mathbb{E}\big[-\log q_t(Y_{t+1})\mid \mathcal{G}_t\big]
\ge
H(p_t)
\qquad \text{a.s.}
\]
Equality holds if and only if
\[
\mathrm{KL}(p_t\|q_t)=0
\qquad \text{a.s.}
\]
On a finite alphabet, $\mathrm{KL}(p_t\|q_t)=0$ is equivalent to equality of the two distributions on the support of $p_t$, and if $q_t(j)>0$ for all $j$ almost surely, then it is equivalent to
\[
q_t(j)=p_t(j),\qquad \forall j\in[K],
\]
almost surely. Thus the conditional expected log-loss is minimized by $q_t=p_t$, and under strict positivity this minimizer is unique almost surely.
\end{proof}
\subsubsection{Proof of Theorem~\ref{thm:train_aligned}}
\label{app:train_aligned_proof}

\begin{proof}
We begin from the constant-learning-rate specialization of Theorem~\ref{thm:pcgs_pathwise}. For any comparator path $\pi_{1:T}\in[K]^T$,
\begin{equation}
\sum_{t=1}^T \langle w_t,\tilde\ell_t\rangle - \sum_{t=1}^T \tilde\ell_{t,\pi_t}
\;\le\;
\frac{1}{\eta}
\Big[
-\log w_1(\pi_1)
-
\sum_{t=1}^{T-1}\log A_t(\pi_t\to\pi_{t+1})
\Big]
+
\frac{\eta T}{8}.
\label{eq:train_aligned_start}
\end{equation}
Accordingly, it suffices to upper bound the transition-complexity term
\[
-\sum_{t=1}^{T-1}\log A_t(\pi_t\to\pi_{t+1})
\]
in terms of the policy outputs $p_t$ and $q_t$.

Recall that in the theorem under consideration the restart intensity is parameterized as
\[
\rho_t=\rho_{\max}p_t,
\qquad p_t\in(0,1),
\]
with $\rho_{\max}\le \tfrac12$. The transition kernel is therefore
\[
A_t(i\to j)
=
(1-\rho_t)\mathbf{1}\{i=j\}+\rho_t q_t(j)
=
(1-\rho_{\max}p_t)\mathbf{1}\{i=j\}+\rho_{\max}p_t\,q_t(j).
\]
Let
\[
s_t(\pi)\triangleq \mathbf{1}\{\pi_{t+1}\neq \pi_t\}
\]
denote the switch indicator of the comparator path at time $t$. We now analyze the transition code length at each time step.

If $s_t(\pi)=1$, then $\pi_{t+1}\neq \pi_t$, so the diagonal term vanishes and the transition probability is exactly
\[
A_t(\pi_t\to\pi_{t+1})
=
\rho_t q_t(\pi_{t+1})
=
\rho_{\max}p_t\,q_t(\pi_{t+1}).
\]
Taking negative logarithms gives the exact identity
\[
-\log A_t(\pi_t\to\pi_{t+1})
=
\log\frac{1}{\rho_{\max}}-\log p_t-\log q_t(\pi_{t+1}).
\]
Thus, on switch steps, the transition penalty is expressed precisely in terms of the switch-probability output $p_t$ and the restart distribution mass assigned to the next comparator expert.

If instead $s_t(\pi)=0$, then $\pi_{t+1}=\pi_t$, and therefore
\[
A_t(\pi_t\to\pi_{t+1})
=
(1-\rho_t)+\rho_t q_t(\pi_t)
=
(1-\rho_{\max}p_t)+\rho_{\max}p_t\,q_t(\pi_t).
\]
Since $q_t(\pi_t)\ge 0$, this implies the lower bound
\[
A_t(\pi_t\to\pi_{t+1})
\ge
1-\rho_{\max}p_t.
\]
Because the logarithm is increasing, taking negative logarithms reverses the inequality:
\[
-\log A_t(\pi_t\to\pi_{t+1})
\le
-\log(1-\rho_{\max}p_t).
\]
We now upper bound the right-hand side by a quantity involving only $p_t$. Set
\[
x\triangleq \rho_{\max}p_t.
\]
Since $p_t\in(0,1)$ and $\rho_{\max}\le \tfrac12$, we have
\[
0\le x\le \tfrac12.
\]
For such $x$, the elementary inequality
\[
-\log(1-x)\le \frac{x}{1-x}
\]
holds, and since $x\le \tfrac12$ implies $(1-x)^{-1}\le 2$, we obtain
\[
-\log(1-x)\le \frac{x}{1-x}\le 2x.
\]
Substituting back $x=\rho_{\max}p_t$ yields
\[
-\log(1-\rho_{\max}p_t)\le 2\rho_{\max}p_t.
\]
We next compare $p_t$ with $-\log(1-p_t)$. Since $p_t\in(0,1)$, the function $u\mapsto -\log(1-u)$ dominates $u$ on $(0,1)$; equivalently,
\[
p_t\le -\log(1-p_t).
\]
Combining the previous two displays gives
\[
-\log(1-\rho_{\max}p_t)
\le
2\rho_{\max}p_t
\le
2\rho_{\max}\big(-\log(1-p_t)\big).
\]
Hence, on stay steps,
\[
-\log A_t(\pi_t\to\pi_{t+1})
\le
2\rho_{\max}\big(-\log(1-p_t)\big).
\]

We may now combine the switch and stay bounds into a single inequality valid for every $t\in\{1,\dots,T-1\}$:
\begin{align*}
-\log A_t(\pi_t\to\pi_{t+1})
\le\;&
s_t(\pi)\Big(\log\tfrac{1}{\rho_{\max}}-\log p_t-\log q_t(\pi_{t+1})\Big)
\\
&\quad
+
(1-s_t(\pi))\,2\rho_{\max}\big(-\log(1-p_t)\big).
\end{align*}
Summing over $t=1,\dots,T-1$ yields
\begin{align}
-\sum_{t=1}^{T-1}\log A_t(\pi_t\to\pi_{t+1})
\le\;&
\sum_{t=1}^{T-1}
s_t(\pi)\Big(\log\tfrac{1}{\rho_{\max}}-\log p_t-\log q_t(\pi_{t+1})\Big)
\notag\\
&\quad
+
2\rho_{\max}\sum_{t=1}^{T-1}(1-s_t(\pi))\big(-\log(1-p_t)\big).
\label{eq:train_aligned_transition_bound}
\end{align}

Finally, substitute \eqref{eq:train_aligned_transition_bound} into \eqref{eq:train_aligned_start}. This gives
\begin{align*}
\sum_{t=1}^T \langle w_t,\tilde\ell_t\rangle - \sum_{t=1}^T \tilde\ell_{t,\pi_t}
\le\;&
\frac{1}{\eta}
\Big[
-\log w_1(\pi_1)
+
\sum_{t=1}^{T-1}
s_t(\pi)\Big(\log\tfrac{1}{\rho_{\max}}-\log p_t-\log q_t(\pi_{t+1})\Big)
\\
&\qquad\qquad
+
2\rho_{\max}\sum_{t=1}^{T-1}(1-s_t(\pi))\big(-\log(1-p_t)\big)
\Big]
+
\frac{\eta T}{8},
\end{align*}
which is exactly \eqref{eq:train_aligned_bound}. This completes the proof.
\end{proof}

\subsection{H. Bounded losses via scaling/clipping: fully precise statement}
\label{app:clipping}

Our regret theorem requires losses in $[0,1]$. We formalize how scaling yields a correct theorem on the original scale.

\begin{proposition}[Scaling to the unit interval and translating the regret bound]
\label{prop:scaling}
Suppose the raw losses satisfy
\[
\ell^{\mathrm{raw}}_{t,k}\in[0,C],
\qquad t\in[T],\; k\in[K],
\]
for some known constant $C>0$. Define the normalized losses by
\[
\ell_{t,k}\triangleq \frac{\ell^{\mathrm{raw}}_{t,k}}{C},
\]
so that $\ell_{t,k}\in[0,1]$ for all $t,k$. Run Generalized Share on the normalized loss sequence $\{\ell_t\}_{t=1}^T$ with learning rate $\eta>0$. Then, for every comparator path $\pi\in[K]^T$,
\[
\sum_{t=1}^T \langle w_t,\ell^{\mathrm{raw}}_t\rangle
-
\sum_{t=1}^T \ell^{\mathrm{raw}}_{t,\pi_t}
\;\le\;
\frac{C}{\eta}
\Big[
-\log w_1(\pi_1)
-
\sum_{t=1}^{T-1}\log A_t(\pi_t\to\pi_{t+1})
\Big]
+
\frac{\eta C T}{8}.
\]
Equivalently, if one introduces the raw-scale learning-rate parameter
\[
\eta_{\mathrm{raw}}\triangleq \frac{\eta}{C},
\]
then the same inequality may be written as
\[
\sum_{t=1}^T \langle w_t,\ell^{\mathrm{raw}}_t\rangle
-
\sum_{t=1}^T \ell^{\mathrm{raw}}_{t,\pi_t}
\;\le\;
\frac{1}{\eta_{\mathrm{raw}}}
\Big[
-\log w_1(\pi_1)
-
\sum_{t=1}^{T-1}\log A_t(\pi_t\to\pi_{t+1})
\Big]
+
\frac{\eta_{\mathrm{raw}} C^2 T}{8}.
\]
\end{proposition}

\begin{proof}
Because $\ell^{\mathrm{raw}}_{t,k}\in[0,C]$, the normalized losses
\[
\ell_{t,k}=\frac{\ell^{\mathrm{raw}}_{t,k}}{C}
\]
satisfy $\ell_{t,k}\in[0,1]$ for every $t\in[T]$ and $k\in[K]$. Therefore the assumptions of Theorem~\ref{thm:pcgs_pathwise} apply to the normalized loss sequence. For any comparator path $\pi\in[K]^T$, Theorem~\ref{thm:pcgs_pathwise} yields
\begin{equation}
\sum_{t=1}^T \langle w_t,\ell_t\rangle
-
\sum_{t=1}^T \ell_{t,\pi_t}
\;\le\;
\frac{1}{\eta}
\Big[
-\log w_1(\pi_1)
-
\sum_{t=1}^{T-1}\log A_t(\pi_t\to\pi_{t+1})
\Big]
+
\frac{\eta T}{8}.
\label{eq:normalized_regret_bound}
\end{equation}

We now rewrite the left-hand side in terms of the raw losses. Since the scaling is deterministic and uniform across experts and time,
\[
\ell^{\mathrm{raw}}_{t,k}=C\,\ell_{t,k}
\qquad\text{for all } t,k.
\]
Hence, for the learner's mixed loss at time $t$,
\[
\langle w_t,\ell^{\mathrm{raw}}_t\rangle
=
\sum_{k=1}^K w_t(k)\ell^{\mathrm{raw}}_{t,k}
=
\sum_{k=1}^K w_t(k)\,C\ell_{t,k}
=
C\langle w_t,\ell_t\rangle.
\]
Likewise, for the comparator path loss,
\[
\ell^{\mathrm{raw}}_{t,\pi_t}=C\,\ell_{t,\pi_t}.
\]
Therefore
\[
\sum_{t=1}^T \langle w_t,\ell^{\mathrm{raw}}_t\rangle
-
\sum_{t=1}^T \ell^{\mathrm{raw}}_{t,\pi_t}
=
C\Big(
\sum_{t=1}^T \langle w_t,\ell_t\rangle
-
\sum_{t=1}^T \ell_{t,\pi_t}
\Big).
\]
Multiplying both sides of \eqref{eq:normalized_regret_bound} by $C$ gives
\[
\sum_{t=1}^T \langle w_t,\ell^{\mathrm{raw}}_t\rangle
-
\sum_{t=1}^T \ell^{\mathrm{raw}}_{t,\pi_t}
\;\le\;
\frac{C}{\eta}
\Big[
-\log w_1(\pi_1)
-
\sum_{t=1}^{T-1}\log A_t(\pi_t\to\pi_{t+1})
\Big]
+
\frac{\eta C T}{8},
\]
which is the first claimed inequality.

For the equivalent raw-scale parameterization, define
\[
\eta_{\mathrm{raw}}\triangleq \frac{\eta}{C},
\qquad\text{so that}\qquad
\eta=C\,\eta_{\mathrm{raw}}.
\]
Substituting this identity into the previous bound yields
\[
\frac{C}{\eta}=\frac{1}{\eta_{\mathrm{raw}}},
\qquad
\frac{\eta C T}{8}
=
\frac{(C\eta_{\mathrm{raw}})CT}{8}
=
\frac{\eta_{\mathrm{raw}} C^2 T}{8},
\]
and therefore
\[
\sum_{t=1}^T \langle w_t,\ell^{\mathrm{raw}}_t\rangle
-
\sum_{t=1}^T \ell^{\mathrm{raw}}_{t,\pi_t}
\;\le\;
\frac{1}{\eta_{\mathrm{raw}}}
\Big[
-\log w_1(\pi_1)
-
\sum_{t=1}^{T-1}\log A_t(\pi_t\to\pi_{t+1})
\Big]
+
\frac{\eta_{\mathrm{raw}} C^2 T}{8}.
\]
This is exactly the raw-scale form of the bound.

The proposition shows that scaling the losses to $[0,1]$ does not alter the structure of the regret certificate; it merely rescales the leading complexity term and the quadratic learning-rate term in the expected way. This completes the proof.
\end{proof}

\begin{remark}[Clipping vs scaling]
If raw losses are unbounded (e.g., squared errors under heavy tails), one typically applies a bounded surrogate by combining scaling with clipping. Theorem~\ref{thm:pcgs_pathwise} then certifies regret with respect to the surrogate losses used by the algorithm and the evaluation protocol. This is standard and avoids vacuous guarantees under infinite-variance regimes.
\end{remark}

\bibliographystyle{unsrt}  
\bibliography{references}

\end{document}